\documentclass{article}

\usepackage{multicol}

\usepackage[final, nonatbib]{neurips_2022}

\usepackage[utf8]{inputenc} 
\usepackage[T1]{fontenc}    
\usepackage{hyperref}       
\usepackage{url}            
\usepackage{booktabs}       
\usepackage{amsfonts}       
\usepackage{nicefrac}       
\usepackage{microtype}      
\usepackage{xcolor}         
\usepackage{microtype}
\usepackage{graphicx}
\usepackage{subfigure}
\usepackage{booktabs} 
\usepackage{hyperref}

\usepackage{amsmath}
\usepackage{amssymb}
\usepackage{mathtools}
\usepackage{amsthm}
\usepackage[capitalize,noabbrev]{cleveref}

\theoremstyle{plain}
\newtheorem{theorem}{Theorem}[section]

\newtheorem{lemma}[theorem]{Lemma}

\theoremstyle{definition}

\theoremstyle{remark}

\newcommand{\homo} {{\mathrm{H}}}
\newcommand{\myparagraph}[1]{\textbf{#1}}
\usepackage{algorithm}
\usepackage{algorithmic}
\usepackage{enumitem}
\usepackage[normalem]{ulem}

\title{Neural Approximation of Graph Topological Features}

\author{Zuoyu Yan \\
Wangxuan Institute of Computer Technology \\
Peking University \\
\texttt{yanzuoyu3@pku.edu.cn}
\And 
Tengfei Ma \\
IBM T. J. Watson Research Center \\
\texttt{tengfei.ma1@ibm.com}
\AND 
Liangcai Gao \\
Wangxuan Institute of Computer Technology \\
Peking University \\
\texttt{glc@pku.edu.cn}
\And
Zhi Tang \\
Wangxuan Institute of Computer Technology \\
Peking University \\
\texttt{tangzhi@pku.edu.cn}
\AND
Yusu Wang \\
Hal{\i}c{\i}o\u{g}lu Data Science Institute \\
University of California \\
\texttt{yusuwang@ucsd.edu}
\And
Chao Chen\thanks{Correspondence to Chao Chen, Yusu Wang, and Liangcai Gao} \\
Department of Biomedical Informatics \\
Stony Brook University \\
\texttt{chao.chen.1@stonybrook.edu}
}

\begin{document}

\maketitle

\begin{abstract}
 Topological features based on persistent homology can capture high-order structural information which can then be used to 
augment graph neural network methods. However, computing extended persistent homology summaries remains slow for large and dense graphs and can be a serious bottleneck for the learning pipeline. Inspired by recent success in neural algorithmic reasoning, we propose a novel graph neural network to estimate extended persistence diagrams (EPDs) on graphs efficiently. 
Our model is built on algorithmic insights, and benefits from better supervision and closer alignment with the EPD computation algorithm. 
We validate our method with convincing empirical results on approximating EPDs and downstream graph representation learning tasks. Our method is also efficient; on large and dense graphs, we accelerate the computation by nearly 100 times. 

\end{abstract}

\section{Introduction}
\label{sec:intro}
Graph neural networks (GNNs) have been widely used in various domains with graph-structured data~\cite{wu2020comprehensive, ma2021deep,kipf2016semi,velivckovic2018graph, chami2019hyperbolic}.
Much effort has been made to understand and to improve graph representation power~\cite{xu2018powerful,morris2019weisfeiler, bodnar2021weisfeiler, maron2019provably}.
An intuitive solution is to explicitly inject high order information, such as graph topological/structural information, into the GNN models~\cite{you2019position, li2020distance}. To this end, persistent homology ~\cite{edelsbrunner2000topological,edelsbrunner2010computational}, which captures topological structures (e.g., connected components and loops) and encodes them in a summary called \emph{persistence
diagram (PD)}, have attracted the attention of researchers. 
Indeed, persistence has already been injected to machine learning pipelines for various graph learning tasks~\cite{zhao2019learning,zhao2020persistence,hofer2020graph,carriere2020perslay,chen2021topological,yan2021link}.
In particular, it has been found helpful to use the so-called \emph{extended persistence diagrams (EPDs)} \cite{cohen2009extending}, which contain richer information than the standard PDs.


\begin{figure*}[btp]
	\centering
	\subfigure[]{
		\begin{minipage}[t]{0.23\linewidth}
			\centering
			\includegraphics[width=\columnwidth]{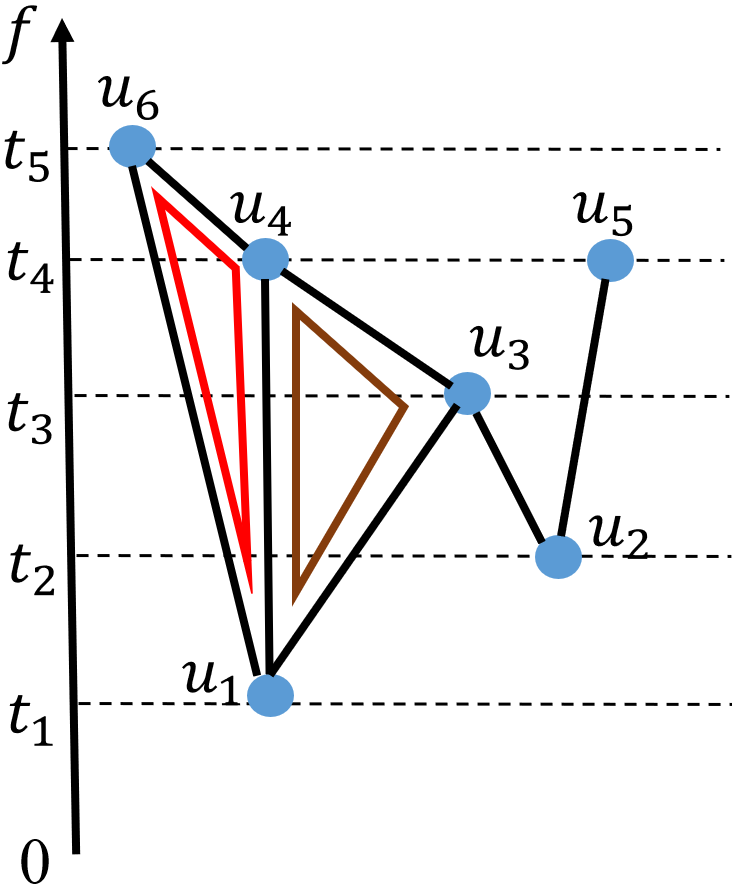}
		\end{minipage}
	}%
	\subfigure[]{
		\begin{minipage}[t]{0.31\linewidth}
			\centering
			\includegraphics[width=\columnwidth]{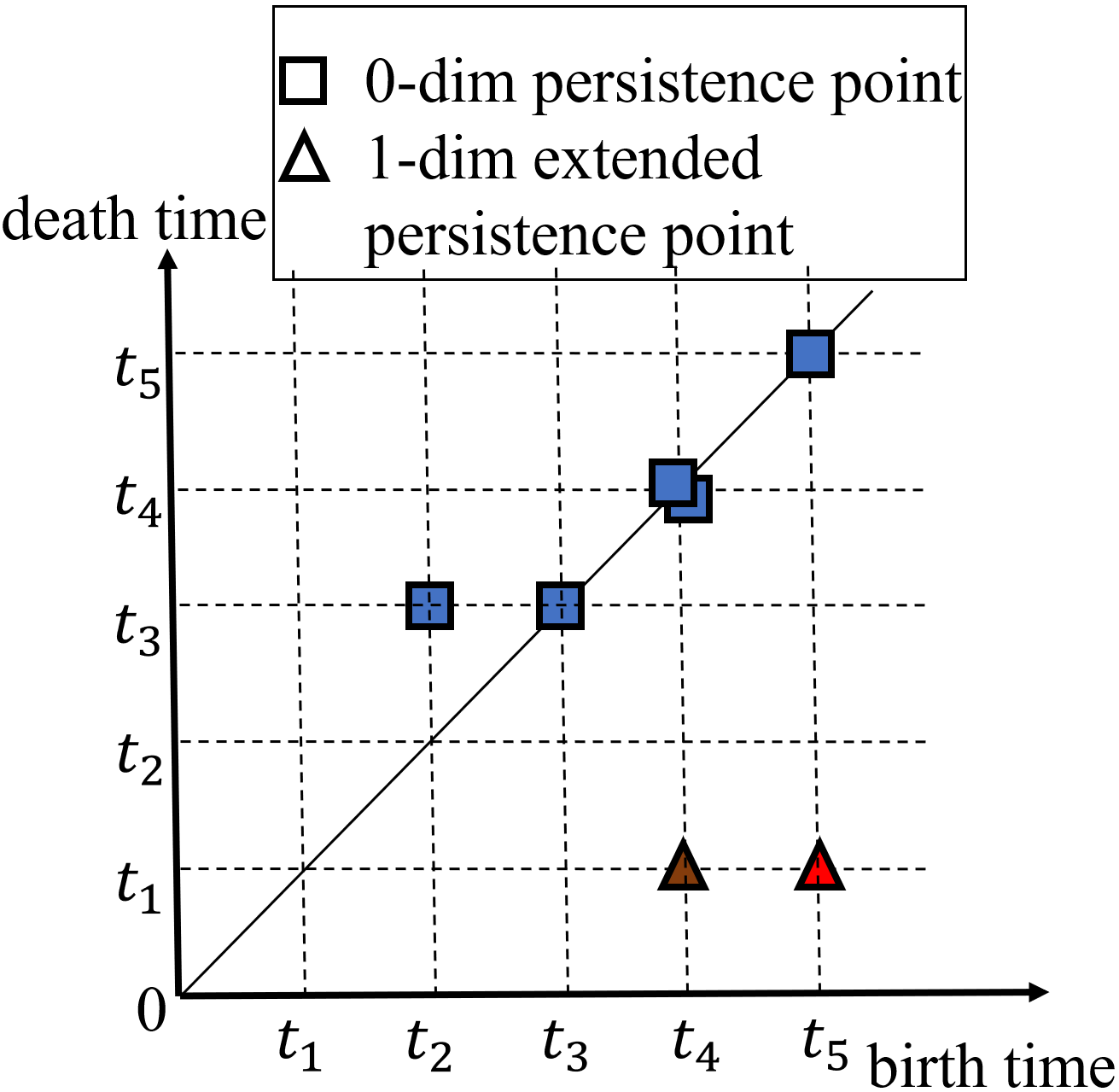}
		\end{minipage}%
	}%
	\subfigure[]{
		\begin{minipage}[t]{0.23\linewidth}
			\centering
			\includegraphics[width=\columnwidth]{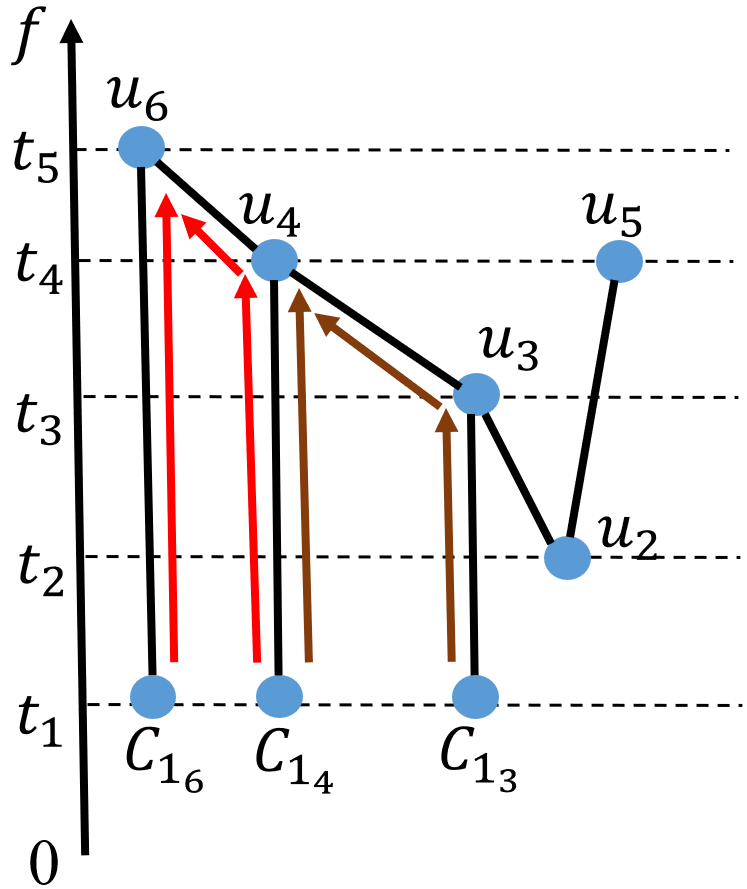}
		\end{minipage}%
	}%
	\subfigure[]{
		\begin{minipage}[t]{0.23\linewidth}
			\centering
			\includegraphics[width=\columnwidth]{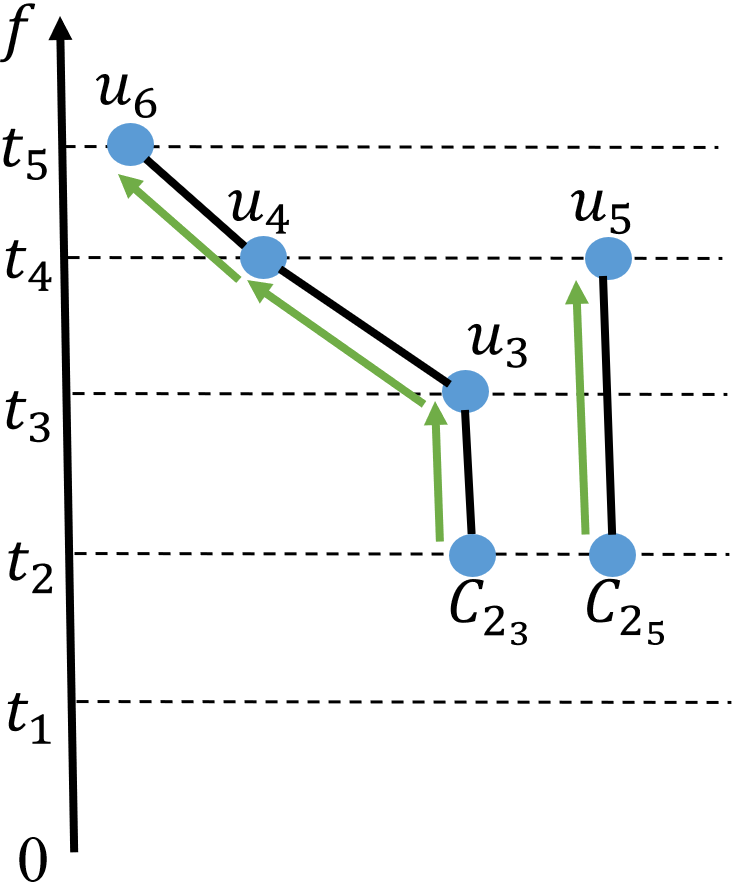}
		\end{minipage}
	}%
	\centering
	\vspace{-.1in}
	\caption{An explanation of extended persistent homology and its computation. The height-function-based filtration is only for illustration purposes. (a) The input graph is plotted with a given filter function. 
	(b) the extended persistence diagram of (a). Commonly speaking, the persistence points on the diagonal (uncritical points) should not be plotted. We plot these points for a clearer illustration. 
	(c) and (d) are examples of finding the loops in the input graph. 
	}
	\vspace{-.15in}
	\label{fig:PD}
	\vspace{-.15in}
\end{figure*}

Despite the usefulness
of PDs and EPDs, their computation remains a bottleneck in graph learning. In situations such as node classification~\cite{zhao2020persistence} or link prediction~\cite{yan2021link}, one has to compute EPDs on vicinity graphs (local subgraph motifs) generated around all the nodes or all possible edges in the input graph. This can be computationally prohibitive 
for large and dense graphs.
Take the Amazon Computers dataset~\cite{shchur2018pitfalls} as an example. To compute EPDs on vicinity graphs take several seconds on average, and there are 13381 nodes. So to compute all EPDs with a single CPU can take up to a day. This is not surprising as, while theoretically EPD for graphs can be computed in $O(n\log n)$ time \cite{agarwal2006extreme}, that algorithm has not been implemented, and practical algorithms for computing PD take quadratic time in worst case~\cite{yan2021link}.

These computational difficulties raise the question: \emph{can we approximate the expensive computation of EPDs using an efficient learning-based approach?}
This is a challenging question due to 
the complex mathematical machinery behind the original algorithm. First, the algorithm involves a reduction algorithm of the graph incidence matrix. Each step of the algorithm is a modulo-2 addition of columns that can involve edges and nodes far apart. Such algorithm can 
be hard to be approximated by a direct application of the black-box deep neural networks. 

The second challenge comes from the supervision. The output EPD is a point set with an unknown cardinality. The distance between EPDs, called the \emph{Wasserstein distance}~\cite{cohen2007stability,cohen2010lipschitz}, involves a complex point matching algorithm. It is 
nontrivial to design a deep neural network with variable output and to support supervision via such Wassserstein distance. Previous attempts~\cite{som2020pi,montufar2020can} directly use black-box neural networks to generate fixed-length vectorization of the PDs/EPDs and use mean squared error or cross-entropy loss for supervision. The compromise in supervision and the lack of control make it hard to achieve high-quality approximation of PDs/EPDs.

In this paper, we propose a novel learning approach to approximate EPDs on graphs. Unlike previous attempts, we address the aforementioned challenges through a carefully designed learning framework guided by several insights into the EPD computation algorithm.

In terms of model output and supervision, we observe that the computation of EPDs can be treated as an edge-wise prediction instead of a whole-graph prediction. 
Each edge in the graph is paired with another graph element (either vertex or edge), and the function values of the pair are the coordinates of a persistence point in the EPD. This observation allows us to compute EPDs by predicting the paired element for every edge of the graph. The Wasserstein distance can be naturally decomposed into supervision loss for each edge. This element-wise supervision can significantly improve learning efficiency compared with previous solutions, which treat PDs/EPDs as a whole-graph representation and have to use whole-graph representation pooling. 

Another concern is whether and how a deep neural network can approximate the sophisticated EPD algorithm. To this end, we redesign the algorithm so that it is better aligned with algorithms that are known to be learnable by neural networks. Recall we observe that computing EPDs can be decomposed into finding pairing for each edge. We show that the decomposition is not only at the output level, but also at the algorithm level. The complex standard EPD computation algorithm can indeed be decomposed into independent pairing problems, each of which can be solved exactly using a classic \emph{Union-Find algorithm}~\cite{cormen2009introduction}. 
To this end, we draw inspiration from recent observations that neural networks can imitate certain categories of sequential algorithms on graphs~\cite{velivckovic2019neural, xhonneux2021transfer}. We propose a carefully designed graph neural network with specific message passing and aggregation mechanism to imitate the Union-Find algorithm. 

Decomposing the algorithm into Union-Find subroutines and approximating them with a customized GNN provide better alignment between our neural network and the EPD algorithm. A better alignment 
can lead to better performance \cite{xu2020can}. 
Empirically, we validate our method by quantifying its approximation quality of the EPDs.  
On two downstream graph learning tasks, node classification and link prediction, we also show that our neural approximations are as effective as the original EPDs. 
Meanwhile, on large and dense graphs, our method is much faster than direct computation. In other words, the approximated EPDs do not lose accuracy and learning power, but can be computed much more efficiently.
Finally, we observe that our model 
can be potentially transferred to unseen graphs, perhaps due to the close imitation of the Union-Find subroutine. 
This is encouraging as we may generalize topological computation to various challenging real-world graphs without much additional 
effort.

In summary, we propose an effective learning approach to approximate EPDs with better supervision and better transparency. The technical contributions are as follows. 
\begin{itemize}[topsep=0pt, partopsep=0pt,itemsep=4pt,parsep=0pt]
    \item We reformulate the EPD computation as an edge-wise prediction problem, allowing better supervision and more efficient representation learning.
    We show that the EPD computation can be decomposed into independent pairing problems, each of which can be solved by the Union-Find algorithm. 
    \item Inspired by recent neural algorithm approximation works~\cite{velivckovic2019neural, xhonneux2021transfer}, we design a novel graph neural network architecture to learn the Union-Find algorithm. The closer algorithmic alignment ensures high approximation quality and transferability.
\end{itemize}

\section{Background: Extended Persistent Homology}
\label{sec:pre}

We briefly introduce extended persistent homology and refer the readers to~\cite{cohen2009extending, edelsbrunner2010computational} for more details. 

\textbf{Ordinary Persistent Homology.} 
Persistent homology captures 0-dimensional (connected components), 1-dimensional (loops) topological structures, as well as high-dimensional analogs, and measures their saliency via a scalar function called \emph{filter function}. Here we will only describe it for the graph setting. Given a input graph $G = (V,E)$, with node set $V$ and edge set $E$, we call all the nodes and edges \emph{simplices}. Denote by $X=V\cup E$ the set of all simplices. We define a filter function on all simpices, $f:X\rightarrow \mathbb{R}$. 
In the typical sublevel-set setting, $f$ is induced by a node-valued function (e.g., node degrees), and further defined on edges as $f(uv) = max(f(u), f(v))$.  

Denote by $X_a$ the \emph{sublevel set of $X$}, consisting of simplices whose filter function values $\leq a$, \mbox{$X_a = \{x\in X|f(x) \leq a\}$}. 
As the threshold value $a$ increases from $-\infty$ to $\infty$, we obtain a sequence of growing spaces, called an \emph{ascending filtration} of $X$: $\emptyset = X_{-\infty} \subset ... \subset X_{\infty} = X.$ As $X_{a}$ increases from $\emptyset$ to $X$, new topological structures gradually appear (born) and disappear (die). For instance, the blue square persistence point at $(t_2, t_3)$ in Figure~\ref{fig:PD} (b) indicates that the connected component $u_2$ appears at $X_{t_2}$ and is merged with the whole connected component at $X_{t_3}$.

Applying the homology functor to the filtration, we can more precisely quantify the birth and death of topological features (as captured by homology groups) throughout the filtration, and the output is the so-called \emph{persistence diagram (PD)}, which is a planar multiset of points, each of which $(b, d)$ corresponds to the birth and death time of some homological feature (i.e., components, loops, and their higher dimensional analogs).  The lifetime $|d - b|$ is called the \emph{persistence} of this feature and intuitively measures its importance w.r.t.~the input filtration. 

\textbf{Extended Persistent Homology.}  In the ordinary persistent homology, topology of the domain (e.g., the graph) will be created at some time (has a birth time), but never dies (i.e., with death time being equal to $+\infty$). We call such topological features \emph{essential features}. 
In the context of graphs, the importance of 
1D essential features, corresponding to independent loops, are not captured via the ordinary persistence. 
To this end, an \emph{extended persistence module} is introduced in \cite{cohen2009extending}: $\emptyset = \homo(X_{-\infty}) \to \cdots \homo(X_a) \to \cdots \homo(X) = \homo(X, X^{\infty}) \to \cdots \to \homo(X, X^a) \to \cdots \to \homo(X, X^{-\infty})$, where \mbox{$X^a = \{x\in X|f(x) \geq a\}$} is a \emph{superlevel set} of $X$ at value $a$. We say that the second part $ \homo(X, X^{\infty}) \to \cdots \to \homo(X, X^a) \to \cdots \to \homo(X, X^{-\infty})$ is induced by a \emph{descending filtration}. If we inspect the persistence diagram induced by this extended sequence, as $\homo(X, X^{-\infty})$ is trivial, all the loop features created will also be killed in the end, and thus captured by persistence points whose birth happens in the ascending filtration and death happens in the descending filtration. In what follows, we abuse the notation slightly and use \emph{1D EPD} to refer to only such persistence points (i.e., born in ascending portion and death in descending portion) in the persistence diagram induced by the extended module\footnote{We  note that in standard terminology, extended persistence diagram will also contain persistent points born and destroyed both in the descending sequence.}. We use 0D PD to refer to the standard ordinary 0D persistence diagram induced by the ascending sequence. Our goal is to compute/approximate the union of 0D PD and 1D EPD. 

Specifically, in the graph setting, at the end of the ascending filtration, some edges, which are the so-called negative edges (as they kill homological features), are paired with the vertices. These correspond to points in the 0D PD, capturing the birth and death of connected components in the ascending filtration. Those unpaired edges, called \emph{positive edges}, will create independent loops (1D homology for graphs) and remain unpaired after the ascending filtration. The number of such unpaired edges equals to the 1st Betti number $\beta_1$ (rank of the 1st homology group). These edges will then be paired in the descending part of the persistence module and their birth-depth times give rise to 1D EPD. An example is given in Figure~\ref{fig:PD}(b). 
Note that since our domain is a graph, $\beta_1$ equals the number of independent loops, which also equals to $\beta_1 = |E|-|V|+1$ for a connected graph. Hence we also say that 1D EPD captures the birth and death of independent loop features. The birth and death times of the loop feature correspond to the threshold value $a$'s when these events happen.
In general, the death time for such loop feature is smaller than the birth time. 
For example, the red triangle persistence point in Figure~\ref{fig:PD} (b) denotes that the red cycle in Figure~\ref{fig:PD} (a) appears at $X_{t_5}$ in the ascending filtration and appears again at $X^{t_1}$ in the descending filtration.

Finally, PDs live in an infinite-dimensional space equipped with an appropriate metric structure, such as the so-called $p$-th Wasserstein distance \cite{cohen2010lipschitz} or the bottleneck distance \cite{cohen2007stability}. They have been combined with various deep learning methods including kernel machines~\cite{reininghaus2015stable, kusano2016persistence, carriere2017sliced}, convolutional neural networks~\cite{hofer2017deep, hu2019topology, wang2020topogan, zheng2021topological}, transformers~\cite{zeng2021topological}, connectivity loss~\cite{chen2019topological, hofer2019connectivity}, and GNNs~\cite{zhao2020persistence,chen2021topological, yan2021link, zhao2019learning, hofer2020graph, carriere2020perslay}. During learning, there have been many works in the literature to vectorize persistence diagrams for downstream analysis. Among these works a popular choice is the persistence image~\cite{adams2017persistence}.

\section{Algorithm Revision: Decomposing EPD into Edge-Wise Paring Predictions}
\label{subsec:compute}

In this section, we provide algorithmic insights into how the expensive and complex computation of EPDs can be decomposed into pairing problems for edges. And each pairing problem can be solved exactly using a Union-Find algorithm. The benefit is two-folds. First, the decomposition makes it possible to train the neural network through edge-wise supervision.
This allows us to adopt the popular and effective edge-prediction GNN for the goal. Second, we observe the similarity between the Union-Find and sequential algorithms which 
are known to be imitable by neural networks. This gives us the opportunity to design a special graph neural network to imitate the algorithm accurately, and to approximate EPDs accurately. 

\textbf{Decompose the EPD Computation into Pairing Computations.} 
Recall that our goal is to compute the 0D PDs and 1D EPDs $PD_0$ and $PD_1$. The reason for not estimating 0D EPDs (or not including the global max/min pair that corresponds to the whole connected component) is that (1) the global max/min value is easy to obtain, and does not need an extra prediction; (2) in our setting, the global max/min pair will not be paired with any edge in the ascending filtration. In the later section, the estimation of EPDs denote the estimation of $PD_0$ and $PD_1$. 

We observe that on these diagrams, each point corresponds to a unique pairing of graph elements (vertex-edge pair for $PD_0$, edge-edge pair for $PD_1$). Each pair of elements are essentially the ``creator'' and ``destroyer'' of the corresponding topological feature during the filtration. And their filtration values are the birth and death times of the topological feature. For example, the persistence point located at $(t_2, t_3)$ in Figure~\ref{fig:PD} (b) denotes that the edge $u_2u_3$ is paired with $u_2$. 
We consider the following "unique pairing" for all edges in the graph: Consider each edge in the ascending filtration: if the edge is a destroyer in the ascending filtration, it will be paired with a vertex. Otherwise, this edge $e$ is a creator in the ascending filtration and will be paired during the descending filtration with another edge $e'$.
We note that this is not in conflict with the fact that the PDs/EPDs are often sparse. Many pairings are local and only pair adjacent elements. They correspond to zero-persistence points living in the diagonal of the diagrams.


This pairing view gives us the opportunity to transform the computation of EPDs into a pairing prediction problem: for every edge in the graph, we predict its pairing element. This will be the foundation of our design of the GNN in Sec.~\ref{sec:model}. 
Meanwhile, we observe that \emph{the decomposition is not only at the output level}. The original algorithm of EPD, a sequential modulo-2 matrix reduction algorithm, can indeed be rewritten into a set of independent algorithm subroutines, each for the computation of one pairing. Each subroutine is a Union-Find algorithm. This new decomposed EPD algorithm has not been reported before, although the idea follows from existing work \cite{agarwal2006extreme}. For completeness, we will provide a proof of correctness of the algorithm. 

\begin{figure}
    \begin{minipage}{\linewidth}
	\begin{algorithm}[H]
	\begin{multicols}{2}
		\begin{algorithmic}[1]	
		\STATE {\bfseries Input:} graph $G = (V,E)$, filter function $f$.
		\STATE Initialise-Nodes($V, f$)
		\STATE $Q = \text{Sort-Queue}(V)$
		\WHILE{Q is not empty}
		\STATE $u = Q.\text{pop-min}()$
		\FOR{$v \in G.\text{neighbors}(u)$}
		\STATE Relax-Edge($u,v,f$)
		\ENDFOR
		\ENDWHILE
		\end{algorithmic}
		\caption{Sequential algorithm} 
		\vspace{-.15in}
		\label{alg:seq}
		\end{multicols}
	\end{algorithm}
	\end{minipage}
\vspace{-.15in}
\begin{minipage}{\linewidth}
\vspace{-.15in}
\begin{algorithm}[H]
    \begin{multicols}{2}
	\begin{algorithmic}[1]
	\STATE {\bfseries Input:} filter function $f$, input graph $G = (V,E)$
	\STATE $V, E = \text{sorted} (V, E, f)$
	\STATE $PD_0 = \text{Union-Find}(V, E, f)$, $PD_1 = \{\}$
	\FOR{$i \in V$}
	\STATE $C_i = \{C_{i_j} | (i, j) \in E, f(j) > f(i)\}$, $E_i = E$
    \FOR{ $C_{i_j} \in C_i$}
	\STATE $f(C_{i_j}) = f(i)$, $E_i = E_i - \{(i,j)\} + \{(C_{i_j}, j)\}$
	\ENDFOR
	\STATE $PD_1^i = \text{Union-Find-step}(V + C_i - \{i\}, E_i, f, C_i)$
	\STATE $PD_1 += PD_1^i$
	\ENDFOR
	\STATE {\bfseries Output:} $PD_0$, $PD_1$
	\end{algorithmic}
    \caption{Computation of EPD}
    \vspace{-.2in}
	\label{alg:ext_PD}
	\end{multicols}
	\vspace{-.15in}
    \end{algorithm}

\end{minipage}
\vspace{-.15in}
\begin{minipage}{\linewidth}
		\begin{algorithm}[H]
		\begin{multicols}{2}
		\begin{algorithmic}[1]
		\STATE {\bfseries Input:} $V$, $E$, $f$, $C_i$ 
		\STATE $PD_1^i = \{\}$
		\FOR{$v \in V$}
		\STATE $v.value = f(v)$, $v.root = v$ 
		\ENDFOR
		\STATE $Q = \text{Sort}(V)$, $Q = Q - \{v|f(v) < f(i)\}$, $G = \{Q, E_Q\}$, where $E_Q = E \cup Q^2$.
		\WHILE{Q is not empty}
		\STATE $u = Q.\text{pop-min}()$
		\FOR{$v \in G.\text{neighbors}(u)$}
		\STATE 
		\STATE $pu, pv = \text{Find-Root}(u), \text{Find-Root}(v)$
		\IF{$pu \neq pv$}
		\STATE $s = argmin(pu.value, pv.value)$
		\STATE $l = argmax(pu.value, pv.value)$
		\STATE $l.root = s$ 
		\IF{$pu \in C_i$ and $pv \in C_i$} 
		\STATE $PD_1^i + \{(u.value, l.value)\}$
		\ENDIF 
		\ENDIF
		\ENDFOR
		\ENDWHILE
		\STATE {\bfseries Function:} $\text{Find-Root} (u)$
		\STATE $pu = u$
		\WHILE {$pu \neq pu.root$}
		\STATE $pu.root = (pu.root).root$, $pu = pu.root$
		\ENDWHILE
		\STATE {\bfseries Return:} $pu$
		\end{algorithmic}
		\vspace{-.2in}
		\caption{Union-Find-step (Sequential)}
		\label{alg:UFs}
		\end{multicols}
		\vspace{-.15in}
		\end{algorithm}
		\end{minipage}
		\vspace{-.15in}
\end{figure}

\textbf{Description of Algorithm~\ref{alg:ext_PD}.} 
The pseudocode for 1D EPD computation is shown in Algorithm~\ref{alg:ext_PD}. We leave the algorithm for 0D PD to the supplementary material\footnote{The 0D algorithm needs a single run of Union-Find~\cite{edelsbrunner2010computational,dey2022computational}, and is very similar to Algorithm \ref{alg:UFs} which is a subroutine used by Algorithm~\ref{alg:ext_PD}.}.
For simplicity of presentation, we assume that all vertices have distinct function values $f: V \to \mathbb{R}$\footnote{We can add jitter to the original filter function. The output EPDs will only have minor changes~\cite{cohen2007stability}}. Therefore finding the persistence value equals to finding the pairing.
To compute the EPD, we traverse all nodes in the vertex set and find their extended persistence pairing. 
Combining the persistence pair from all nodes, we can obtain the final EPD. The algorithm complexity analysis is provided in the supplementary material. 

\textbf{Finding persistence pairing for nodes.} For node $u_i \in V$, we can call Algorithm \ref{alg:UFs} to identify the corresponding persistence pair. 
In particular, the algorithm first sorts the graph elements according to an input scalar function, then does the edge operation by finding the roots of the corresponding nodes and merging these nodes. 
See Figure~\ref{fig:PD}(c) for a simple illustration. For node $u_1$, there are three upper edges: $u_1u_3$, $u_1u_4$, and $u_1u_6$.
We put each such edge $u_iu_j$ in a different component $C_{i_j}$, -- we call this \emph{upper-edge splitting operation} -- and start to sweep the graph in increasing values starting at $f(u_i)$. Then, the first time any two such components merge will give rise to a new persistence point in the 1D EPD. For instance, $C_{1_4}$ and $C_{1_3}$ first merge at $u_4$, and this will give rise to the brown loop in Figure~\ref{fig:PD}(a) with $(t_4, t_1)$ as its persistence point. While in Figure~\ref{fig:PD} (d), the two connected components, $C_{2_3}$ and $C_{2_5}$ (originated from $u_2$) will not be united. Therefore, node $u_2$ will not lead to any persistence point in the EPD. 

\textbf{Correctness.} 
The idea behind Algorithm~\ref{alg:ext_PD} to compute the extended pairing for essential edges appears to be folklore. For completeness, we provide a proof of its correctness (stated in Theorem~\ref{the:1}). 
We provide a sketch of the proof here, leaving the complete proof to the supplementary material. 

\begin{theorem}
\label{the:1}
Algorithm~\ref{alg:ext_PD} outputs the same 1D EPDs as the standard EPD computation algorithm.
\end{theorem}
\emph{Proof sketch.} To compute the 1D EPDs, we simply need to find the pairing partner for all edges. Therefore, to prove that the two algorithms output the same 1D EPDs, we need to prove that the output pairing partners are the same (or share the same filter value). We prove this by showing that both the standard EPD computation algorithm and Algorithm~\ref{alg:ext_PD} find the ``thinnest pair", i.e., the paired saddle points are with the minimum distance in terms of filter value, for all edges.

\textbf{Neural Approximation of Union-Find.}
In the previous paragraph, we showed that the computation of 1D EPDs can be decomposed into the parallel execution of Union-Find algorithms, which share a similar sequential behavior. This gives us the opportunity to approximate these Union-Find algorithms well, and consequently approximate EPDs well.

Approximating algorithms with neural networks is a very active research direction~\cite{zaremba2014learning, kaiser2015neural, kurach2015neural, reed2015neural, santoro2018relational, yan2020neural}. Within the context of graph, GNNs have been proposed to approximate parallel algorithms (e.g., Breadth-First-Search) and sequential algorithms (e.g., Dijkstra)~\cite{velivckovic2019neural, velivckovic2020pointer,xhonneux2021transfer}. Particularly relevant to us is the success in approximating the category of sequential algorithms such as Dijkstra. These sequential algorithms, as generally defined in Algorithm~\ref{alg:seq}, sort graph elements (vertices and edges) according to certain function, and perform algorithmic operations according to the order. As described in previous paragraphs, the Union-Find algorithm also contains these steps, and can be expressed in a sequential-like form (Algorithm~\ref{alg:UFs}). Therefore we propose a framework to simulate the algorithm. 
\section{A Graph Neural Network for EPD Approximation}
\label{sec:model}
Previous section establishes the algorithm foundation by showing that we can decompose EPD computation into edge pairing prediction problems, each of which can be solved using a Union-Find algorithm.
Based on such algorithmic insights, we next introduce our neural network architecture to approximate the EPDs on graphs. Our main contributions are: (1) we transform the EPD computation into an edge-wise prediction problem, and solve it using a GNN framework, inspired by the GNN for link prediction; (2) we design a new backbone GNN model \textbf{PDGNN} to approximate the Union-Find algorithm, with specially designed pooling and message passing operations.


\subsection{EPD computation as a edge-wise prediction problem}
\vspace{-.1in}

\begin{figure}[btp]
	\centering
	\scalebox{0.7}{
	\includegraphics[width=\columnwidth]{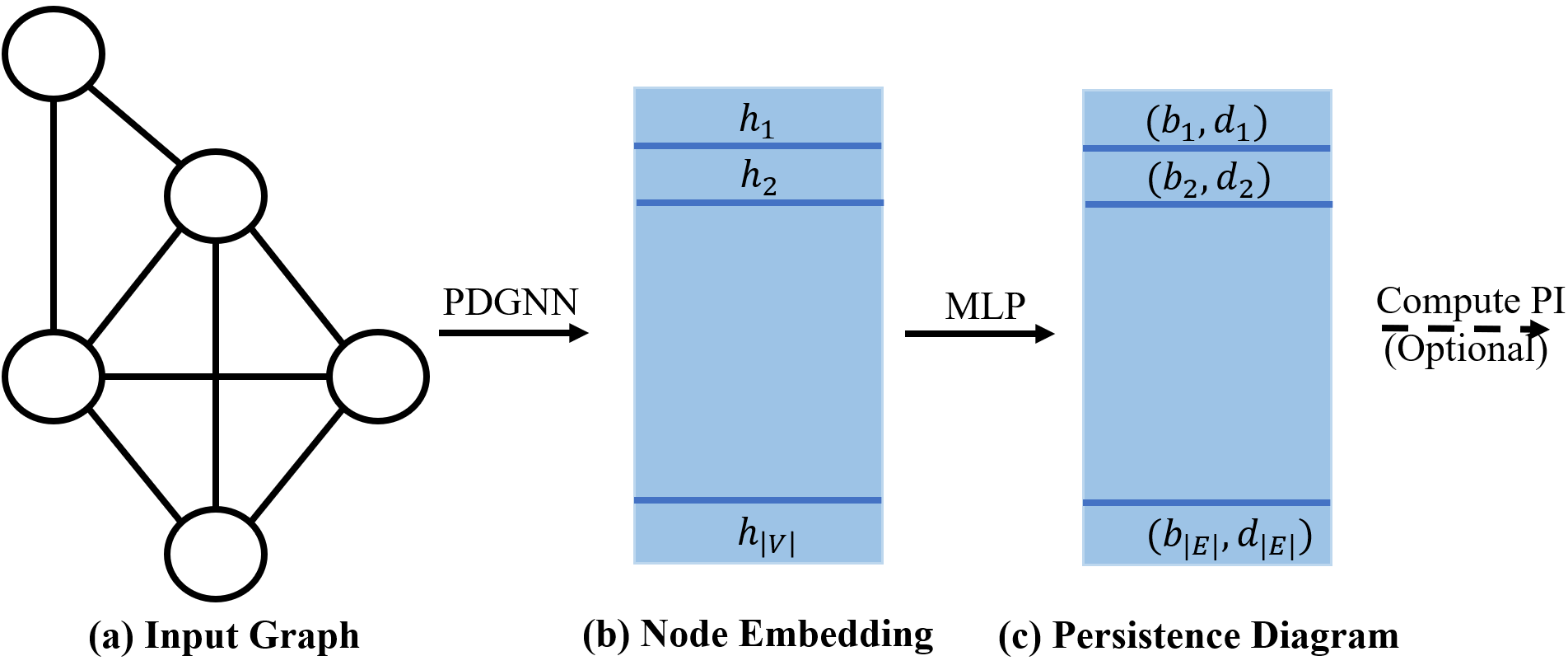}
	}
	\vspace{-.1in}
	\caption{The basic framework.}
	\label{fig:PDGNN}
	\vspace{-.15in}
\end{figure}

We have established that computing $PD_0$ and $PD_1$ can be reduced into finding the pairing partners for all edges.
We transfer the problem into an edge-wise prediction problem. We predict the persistence pairing for all edges. 
This is very similar to a standard link prediction problem \cite{chami2019hyperbolic, yan2021link}, in which one predicts for each node pair of interest whether it is a real edge of the graph or not. 

Inspired by standard link-prediction GNN architectures~\cite{chami2019hyperbolic, yan2021link}, we propose our model (see Figure~\ref{fig:PDGNN}) as follows. 
(1) For an input graph $G = (V,E)$ and a filter function $f$, we first obtain the initial filter value for all the nodes: $X = f(V) \in R^{|V|*1}$, and then use a specially designed GNN model which later we call PDGNN $\mathcal{G}$ 
to obtain the node embedding for all these vertices: $H = \mathcal{G}(X) \in R^{|V|*d_H}$. 
(2) Subsequently, a MLP (Multi-layer perceptron) $W$ is applied to the node embeddings to obtain a two dimensional output for each edge $(u,v)\in E$, corresponding to its persistence pairing. Formally, we use $PP_{uv} = W([h_u \bigoplus h_v]) \in R^2$ as the persistence pair. 
Here, $h_u$ and $h_v$ denote the node embedding for node $u$ and $v$, and $\bigoplus$ represents the concatenation of vectors. 

In Algorithm~\ref{alg:ext_PD}, the Union-Find-step should be implemented on all edges to obtain 1D EPDs. Hence ideally we would need a large GNN model with node features proportional to the graph size so as to simulate all these Union-find-steps in parallel simultaneously. However this would be expensive in practice.  On the other hand, there are many overlapping or similar computational steps between the Union-Find-step procedures on different vertices. Hence in practice, we only use bounded-size node features.

\subsection{PDGNN}
\vspace{-.1in}
In this section, we explain how to design the backbone GNN to approximate the Union-Find algorithm. Note the Union-Find is similar to known sequential algorithms but with a few exceptions. We design specific pooling and message passing operations to imitate these special changes. These design choices will be shown to be necessary in the experiment section.

Recall a typical GNN learns the node embedding via an iterative aggregation of local graph neighbors. Following~\cite{xu2018powerful}, we write the $k$-th iteration (the $k$-th GNN layer) as: 
\begin{equation}
\label{equa:GNN}
h_u^{k} = AGG^k(\{MSG^k(h_v^{k-1}), v \in N(u)\}, h_u^{k-1})
\end{equation}
where $h_u^k$ is the node features for node $u$ after $k$-th iterations, 
and $N(u)$ is the neighborhood of node $u$.
In our setting, $h_u^0 = x_u$ is initialized to be the filter value of node $u$. 
Different GNNs have different $MSG$ and $AGG$ functions, e.g., in GIN~\cite{xu2018powerful}, the message function $MSG$ is a MLP followed by an activation function, and the aggregation function $AGG$ is a sum aggregation function. 

We now describe our specially designed GNN, called \textbf{PDGNN} (Persistence Diagram Graph Neural Network). Compared with the Sequential algorithms (Algorithm \ref{alg:seq}) \cite{xhonneux2021transfer}, our Union-Find algorithm (Algorithm \ref{alg:UFs}) differs in: (1) the Find-Root algorithm which needs to return the minimum of the component, (2) additional edge operations such as upper-edge splitting. 
To handle these special algorithmic needs, our PDGNN modifies standard GNNs with the following modules. 

\textbf{A new aggregation due to the Find-Root function.} 
Finding the minimum intuitively suggests using a combination of several local min-aggregations. Considering that the sum aggregation can bring the best expressiveness to GNNs \cite{xu2018powerful}, we implement the root-finding process by a concatenation of sum aggregation and min aggregation as our aggregation function. To be specific: 
\begin{equation}
AGG^k(.) = SUM(.) \bigoplus MIN(.)
\end{equation}

\textbf{Improved edge operations.} 
As shown in~\cite{velivckovic2019neural,xhonneux2021transfer}, classic GNNs are not effective in ``executing'' 
Relax-Edge subroutines. Furthermore, in Algorithm \ref{alg:ext_PD}, we also need the upper-edge splitting operation for each vertex. In other words, the information of the separated components $C_{i_j}$ are formed by the information from both nodes $u_i$ and $u_j$. To this end, 
we use edge features and attention to provide bias using edges. Specifically, we propose the following message function in the $k$-th iteration: 
\begin{equation}
MSG^{k}(h_v^{k-1}) = \sigma^k[\alpha_{uv}^k(h_u^{k-1} \bigoplus h_v^{k-1})W^{k}]
\end{equation}
where $\sigma^k$ is an activation function, $W^k$ is a MLP module, and $\alpha_{uv}^k$ is the edge weight for $uv$. We adopt $PRELU$ as our activation function, and the edge weight proposed in~\cite{velivckovic2018graph} as our edge weight.

\textbf{Training PDGNN.} We use the 2-Wasserstein distance between the predicted diagram and the ground truth EPD as the loss function. Through optimal matching, the gradient is passed to each predicted persistence pair. Since we have established the one-to-one correspondence between pairs and edges, the gradient is then passed to the corresponding edge, and contributes to the representation learning. 


\begin{table*}
	\vspace{-0.15 in}
	\centering
	\caption{Approximation error on different vicinity graphs}
	\label{tab:app_nc} 
	\scalebox{0.8}{
		\begin{tabular}{l|cc|cc|cc|cc|cc}
			\hline\noalign{\smallskip}
			Dataset & \multicolumn{2}{c|}{Cora} & \multicolumn{2}{c|}{ Citeseer} & \multicolumn{2}{c|}{PubMed} & \multicolumn{2}{c|}{Photo} & \multicolumn{2}{c}{Computers}\\
			Evaluation & $W_2$ & PIE  &  $W_2$ & PIE &  $W_2$ & PIE &  $W_2$ & PIE&  $W_2$ & PIE\\
			\noalign{\smallskip}\hline\noalign{\smallskip}
			GIN\_PI &--- & 5.03e-1 &--- &2.17e-1&--- &4.08e-1 & ---&5.53 &---& 2.70\\
			GAT\_PI & ---&1.43e-1& ---&1.95e-1& ---&1.60 &--- &20.98 &---& 44.50\\
			\noalign{\smallskip}\hline\noalign{\smallskip}
			GAT &  0.655&2.46e-2 & 0.431 & 4.04e-2&0.697 &3.5e-1& 1.116&1.09& 1.145 & 2.21\\
			GAT (+MIN) & 0.579 & 1.53e-2  &0.344 & 1.02e-2 & 0.482 & 4.60e-2  & 0.820 & 1.35& 0.834&0.64 \\
			PDGNN (w$\slash$o ew) & 0.692 & 2.77e-2& 0.397&2.24e-2 & 0.666 & 9.01e-2& 2.375&6.47& 18.63 & 27.35\\
			\noalign{\smallskip}
			\hline
			\noalign{\smallskip}
			PDGNN & \textbf{0.241} & \textbf{4.75e-4} & \textbf{0.183} & \textbf{4.43e-4}& \textbf{0.256} & \textbf{8.95e-4} & \textbf{0.224} & \textbf{4.33e-3} & \textbf{0.220} & \textbf{6.20e-3}\\
			\noalign{\smallskip}
			\hline
			\noalign{\smallskip}
	\end{tabular}}
	\vspace{-0.1 in}
\end{table*}

\begin{table*}
	\vspace{-0.1 in}
	\centering
	\caption{Classification accuracy on various node classification benchmarks}
	\label{tab:NC}
	\scalebox{0.8}{
		\begin{tabular}{l|ccccc|cc}
			\hline\noalign{\smallskip}
			Method & Cora & Citeseer &  PubMed & Computers & Photo & CS & Physics \\
			\noalign{\smallskip}\hline\noalign{\smallskip}
			GCN & 81.5$\pm$0.5 & 70.9$\pm$0.5 & 79.0$\pm$0.3 & 82.6$\pm$2.4 & 91.2$\pm$1.2 & 91.1$\pm$0.5 & 92.8$\pm$1.0\\
			GAT & \textbf{83.0$\pm$0.7} & \textbf{72.5$\pm$0.7} & 79.0$\pm$0.3 & 78.0$\pm$19.0 & 85.1$\pm$20.3 & 90.5$\pm$0.6 & 92.5$\pm$0.9\\
			HGCN & 78.0$\pm$1.0 & 68.0$\pm$0.6 & 76.5$\pm$0.6 & 82.1$\pm$0.0 & 90.5$\pm$0.0 & 90.5 $\pm$ 0.0& 91.3$\pm$0.0 \\
			\noalign{\smallskip}
			\hline
			\noalign{\smallskip}
			PEGN (True Diagram)  & 82.7$\pm$0.4 & 71.9$\pm$0.5 & \textbf{79.4$\pm$0.7} & 86.6$\pm$0.6 & \textbf{92.7$\pm$0.4} & \textbf{93.3$\pm$0.3} & \textbf{94.3$\pm$0.1} \\	
			\noalign{\smallskip}
			\hline
			\noalign{\smallskip}
			PEGN (GIN\_PI) & 81.8$\pm$0.1 & 65.7$\pm$2.1 & 77.7$\pm$0.9 & 82.4$\pm$0.5  & 88.3$\pm$0.7 & 92.6$\pm$0.3 & 93.7$\pm$0.5  \\
			PEGN (PDGNN) & 82.0$\pm$0.5 & 70.8$\pm$0.5 & 78.7$\pm$0.6 & \textbf{86.7$\pm$0.9}  & 92.2$\pm$0.2  & 93.2$\pm$0.2 & 94.2$\pm$0.2  \\
			\noalign{\smallskip}\hline
	\end{tabular}}
\end{table*}

\begin{table*}
	\vspace{-0.15 in}
	\centering
	\caption{AUC-ROC score on various link prediction benchmarks}
	\label{tab:LP} 
	\scalebox{0.9}{
		\begin{tabular}{lccccc}
			\hline\noalign{\smallskip}
			Method & Cora & Citeseer & PubMed & Photo & Computers \\
			\noalign{\smallskip}\hline\noalign{\smallskip}
			GCN  & 90.5$\pm$ 0.2 & 82.6$\pm$1.9 & 89.6$\pm$3.7 & 91.8$\pm$0.0 & 87.8$\pm$0.0\\
			GAT  & 72.8$\pm$ 0.2 & 74.8$\pm$1.5 & 80.3$\pm$0.0&  92.9$\pm$0.3& 86.4$\pm$0.0\\
			HGCN & 93.8$\pm$0.1  & \textbf{96.6$\pm$0.1} & 96.3$\pm$0.0* & 95.4$\pm$0.0 & 93.6$\pm$0.0  \\
			P-GNN  &  74.1$\pm$2.4 & 73.9$\pm$2.6 &  79.6$\pm$0.5 & 90.9$\pm$0.7 & 88.3$\pm$1.0\\
			SEAL & 91.3$\pm$5.7 & 89.8$\pm$2.3& 92.4$\pm$1.2 & 97.8$\pm$1.3 & 96.8$\pm$1.5 \\
			\noalign{\smallskip}
			\hline
			\noalign{\smallskip}
			TLC-GNN (True Diagram) & 94.9$\pm$0.4 & 95.1$\pm$ 0.7 & \textbf{97.0$\pm$0.1} & 98.2$\pm$0.1 & 97.9$\pm$0.1\\
			\noalign{\smallskip}\hline
			TLC-GNN (GIN\_PI) & 93.5$\pm$0.2 & 93.3$\pm$0.6  & 96.3 $\pm$ 0.2  & 95.8$\pm$ 1.0 & 96.2$\pm$0.3 \\
			TLC-GNN (PDGNN) &\textbf{95.0$\pm$0.3} & 95.6$\pm$0.4  &  \textbf{97.0$\pm$0.1}   & \textbf{98.4$\pm$0.6}  & \textbf{98.2$\pm$0.3} \\
			\hline
	\end{tabular}}
	\vspace{-0.10 in}
\end{table*}

\begin{table*}
	\vspace{-0.05 in}
	\centering
	\caption{Time evaluation on different datasets (seconds)}
	\label{tab:time} 
	\scalebox{0.9}{
		\begin{tabular}{lccccccc}
			\hline\noalign{\smallskip}
			Dataset & Cora & Citeseer & PubMed & Photo  & Computers & CS & Physics\\
			\noalign{\smallskip}\hline\noalign{\smallskip}
			Avg.~N/E & 38/103 &16/43 & 61/190 & 797/16042 & 1879/47477 & 97/431 & 193/1315\\
			\hline\noalign{\smallskip}
			Fast~\cite{yan2021link} & 0.95 & 0.39 & 2.15 & 362.60 & 1195.66 &5.72 & 24.14 \\
			Gudhi~\cite{gudhi:urm} & \textbf{0.44} & \textbf{0.21} & \textbf{1.00} & 583.55& 8585.50& \textbf{3.00}& 26.58\\
			Ours & 5.21 & 4.72 & 4.78& \textbf{6.67}& \textbf{7.32} & 5.18 & \textbf{5.42}\\
			\noalign{\smallskip}
			\hline
			\noalign{\smallskip}
	\end{tabular}}
	\vspace{-0.15 in}
\end{table*}

\section{Experiments}
\label{sec:exp}
\vspace{-0.1 in}

In this section, we thoroughly evaluate the proposed model from 3 different perspectives. In Section~\ref{subsec:app}, we evaluate the approximation error between the predicted diagram and the original diagram and show that the prediction is very close to the ground truth. Even with a small approximation error, we still need to know how much does the error influence downstream tasks. Therefore, in Section~\ref{subsec:down}, we evaluate the learning power of the predicted diagrams through 2 downstream graph representation learning tasks: node classification and link prediction. We observe that the model using the predicted diagrams performs comparably with the model using the ground truth diagrams. In Section~\ref{subsec:eff}, we evaluate the  efficiency of the proposed algorithm. Experiments demonstrate that the proposed method is much faster than the original algorithm, especially on large and dense graphs. Source code is available at \href{https://github.com/pkuyzy/TLC-GNN}{https://github.com/pkuyzy/TLC-GNN}.

\textbf{Datasets.} To compute EPDs, we need to set the input graphs and the filter functions. Existing state-of-the-art models on node classification~\cite{zhao2020persistence} and link prediction~\cite{yan2021link} mainly focus on the local topological information of the target node(s). Following their settings, for a given graph $G = (V,E)$, we extract the $k$-hop neighborhoods of all the vertices, and extract $|V|$ vicinity graphs. In our experiments, $k$ is set to 1 or 2 (details are provided in the supplementary material). 

In terms of filter functions, we use Ollivier-Ricci curvature~\cite{ni2018network}, heat kernel signature with two temprature values~\cite{sun2009concise, hu2014stable} and the node degree\footnote{Following the settings in \cite{zhao2020persistence,yan2021link},  we adopt the Ollivier-Ricci curvature as the graph metric, and the distance to target node(s) as the filter function; Following the settings in \cite{carriere2020perslay}, we set the temparature $t = 10$ and $0.1$ and adopt these two kernel functions as the filter functions; Node degree is used as the initial filter function in \cite{hofer2020graph}.}. For an input vicinity graph, we compute 4 EPDs based on the 4 filter functions, and then vectorize them to get 4 peristence images~\cite{adams2017persistence}. Therefore, we can get 4$|V|$ EPDs in total. The input graphs include (1) citation networks including Cora, Citeseer, and PubMed~\cite{sen2008collective}; (2) Amazon shopping datasets including Photo and Computers~\cite{shchur2018pitfalls}; (3) coauthor datasets including CS and Physics~\cite{shchur2018pitfalls}. Details 
are available in the supplementary material. 


\subsection{Approximation Quality}
\label{subsec:app}
\vspace{-0.1 in}

In this section, we evaluate the approximation error between the prediction and the original EPDs.

\textbf{Evaluation metrics.} Recall that the input of our model is a graph and a filter function, and the output is the predicted EPD. After obtaining the predicted EPD, we vectorize it with persistence image~\cite{adams2017persistence} and evaluate (1) the 2-Wasserstein ($W_2$) distance between the predicted diagram and the ground truth EPD; (2) the total square error between the predicted persistence image and the ground truth image (persistence image error, denoted as PIE). Considering that our aim is to estimate EPDs on graphs rather than roughly approximating persistence images, we use the $W_2$ distance as the training loss, while the PIE is only used as an evaluation metric. Given an input graph (e.g., Cora, Citeseer, etc.) and a filter function, we extract the $k$-hop neighborhoods of all the vertices and separate these vicinity graphs randomly into 80\%/20\% as training/test sets. We report the mean $W_2$ distance between diagrams and PIE on different vicinity graphs and 4 different filter functions. 

\textbf{Baseline settings.} \textbf{PDGNN} denotes our proposed method, that is, the GNN framework with the proposed $AGG$ function and $MSG$ function. 
Its strategy is to first predict the EPD, and then convert it to the persistence image. To show its superiority, we compare with the strategy from~\cite{som2020pi,montufar2020can}, i.e., directly approximate the persistence image of the input graph, as a baseline strategy. \textbf{GIN\_PI} and \textbf{GAT\_PI} denote the baseline strategy with GIN~\cite{xu2018powerful} and GAT~\cite{velivckovic2018graph} as the backbone GNNs.

To show the effectiveness of the modules proposed in Section~\ref{sec:model}, we add other baselines with our proposed strategy. \textbf{GAT} denotes GAT as the backbone GNN. \textbf{GAT (+MIN)} denotes GAT with the new $AGG$ function. Compared with PDGNN, it exploits the original node feature rather than the new edge feature in the $MSG$ function. \textbf{PDGNN (w$\slash$o ew)} denotes PDGNN without edge weight. Further experimental settings can be found in the supplementary material.

\textbf{Results.} Table~\ref{tab:app_nc} reports the approximation error, we observe that PDGNN outperforms all the baseline methods among all the datasets. The comparison between GAT and GAT\_PI shows the benefit of predicting EPDs instead of predicting the persistence image. Comparing GAT and \mbox{GAT (+MIN)}, we observe the advantage of the new $AGG$ function, which shows the necessity of using min aggregation to approximate the Find-Root algorithm; Comparing GAT (+MIN) and PDGNN, we observe the effectiveness of using the new $MSG$ function to help the model capture information of the separated connected components. 
The comparison between PDGNN (w$\slash$o ew) and PDGNN shows that edge weights help the model focus on the individual Relax-Edge sub-algorithm operated on every edge.

\subsection{Downstream Tasks}
\label{subsec:down}
\vspace{-0.1 in}

In this section, we evaluate the performance of the predicted diagrams on 2 graph representation learning tasks: node classification and link prediction. 
We replace the ground truth EPDs in state-of-the-art models based on 
persistence~\cite{zhao2020persistence, yan2021link} with our predicted diagrams and report the results.

\textbf{Baselines.} We compare our method with various state-of-the-art methods. We compare with popular GNN models including \textbf{GCN}~\cite{kipf2016semi}, \textbf{GAT}~\cite{velivckovic2018graph} and \textbf{HGCN}~\cite{chami2019hyperbolic}. For link prediction, we compare with several state-of-the-art methods such as \textbf{SEAL}~\cite{zhang2018link} and \textbf{P-GNN}~\cite{you2019position}. Notice that GCN and GAT are not originally designed for link prediction, therefore we follow the settings in~\cite{chami2019hyperbolic, yan2021link}, that is, to get the node embedding through these models, and use the Fermi-Dirac decoder~\cite{krioukov2010hyperbolic, nickel2017poincare} to predict whether there is a link between the two target nodes. In comparison with the original EPD, we also add \textbf{PEGN}~\cite{zhao2020persistence} and \textbf{TLC-GNN}~\cite{yan2021link} as baseline methods. Furthermore, to show the benefit of directly predicting EPDs, we also add the baseline methods \textbf{PEGN (GIN\_PI)} and \textbf{TLC-GNN (GIN\_PI)}, which replace the original persistent homology feature with the output from GIN\_PI.

\textbf{Evaluation metrics.} For node classification, our setting is the same as~\cite{kipf2016semi, velivckovic2018graph, zhao2020persistence}. To be specific, we train the GNNs with 20 nodes from each class and validate (resp.~test) the GNN on 500 (resp.~1000) nodes. We run the GNNs on these datasets 10 times and report the average classification accuracy and standard deviation. For link prediction, our setting is the same as~\cite{chami2019hyperbolic, yan2021link}. To be precise, we randomly split existing edges into 85/5/10\% for training, validation, and test sets. An equal number of non-existent edges are sampled as negative samples in the training process. We fix the negative validation and test sets, and randomly select the negative training sets in every epoch. We run the GNNs on these datasets 10 times and report the mean average area under the ROC curve (ROCAUC) scores and the standard deviation.

\textbf{Results.} Table~\ref{tab:NC} and Table~\ref{tab:LP} summarize the performance of all methods on node classification and link prediction. We observe that PEGN (PDGNN) and TLC-GNN (PDGNN) consistently perform comparably with PEGN and TLC-GNN, showing that the EPDs approximated by PDGNN have the same learning power as the true EPDs. Furthermore, PEGN using the approximated EPDs achieve better or comparable performance with different SOTA methods.

We also discover that PEGN (GIN\_PI) and TLC-GNN (GIN\_PI) perform much inferior to the original models using the true EPDs. It demonstrates that the large approximation error from GIN\_PI lose much of the crucial information which is preserved in PDGNN.

\myparagraph{Transferability.}
One appealing feature of our method is its transferability. Training on one graph, our algorithm can estimate EPDs well on another graph. This makes it possible to apply the computationally expensive topological features to a wide spectrum of real-world graphs; we can potentially apply a pre-trained model to large and dense graphs, on which direct EPD computation is infeasible. The experiments are provided in the supplementary material.


\subsection{Algorithm Efficiency}
\label{subsec:eff}

In this section, we evaluate the efficiency of our proposed model.
For a fair and complete comparison, we compare with algorithms from Gudhi~\cite{gudhi:urm} and from~\cite{yan2021link}. We select the first 1000 nodes from Cora, Citeseer, PubMed, Photo, Computers, CS, Physics, and then extract their 2-hop neighborhoods. With Ollivier-Ricci curvature as the filter function, we compute the EPDs and report the time (seconds) used to infer these diagrams.

\textbf{Results.} We list the average nodes and edges of these vicinity graphs in the first line of Table~\ref{tab:time}. As shown in Table~\ref{tab:time}, although our model is slower on small datasets like Cora or Citeseer, it is much faster on large and dense datasets. Therefore we can simply use the original algorithm to compute the EPDs on small graphs, and use our model to estimate EPDs on large graphs. The model can be applied to various graph representation learning works based on persistent homology.

\section{Conclusion}

Inspired by recent success on neural algorithm execution, we propose a novel GNN with different technical contributions to simulate the computation of EPDs on graphs. The network is built on algorithmic insights, and benefits from better supervision and closer alignment with the EPD computation algorithm. Experiments show that our method achieves satisfying approximation quality and learning power while being significantly faster than the original algorithm on large and dense graphs. Another strength of our method is the transferability: training on one graph, our algorithm can still approximate EPDs well on another graph. This makes it possible to apply the computationally expensive topological features to a wide spectrum of real-world graphs. 

\textbf{Acknowledgements.}
We thank all anonymous reviewers for their constructive feedback very much. This work of Zuoyu Yan, Liangcai Gao, and Zhi Tang is supported by the projects of National Key R\&D Program of China (2019YFB1406303)
and National Natural Science Foundation of China (No.~61876003), which is also a research achievement of Key Laboratory of Science, Technology and Standard in Press Industry (Key Laboratory of Intelligent Press Media Technology).

\bibliography{reference.bib}
\bibliographystyle{plain}




\section*{Checklist}


\begin{enumerate}

\item For all authors...
\begin{enumerate}
  \item Do the main claims made in the abstract and introduction accurately reflect the paper's contributions and scope?
    \answerYes{}
  \item Did you describe the limitations of your work?
    \answerYes{It is discussed in the supplementary material.}
  \item Did you discuss any potential negative societal impacts of your work?
    \answerNo{The work does not seem to have any negative societal impact.}
  \item Have you read the ethics review guidelines and ensured that your paper conforms to them?
    \answerYes{}
\end{enumerate}

\item If you are including theoretical results...
\begin{enumerate}
  \item Did you state the full set of assumptions of all theoretical results?
    \answerNA{}
        \item Did you include complete proofs of all theoretical results?
    \answerNA{}
\end{enumerate}

\item If you ran experiments...
\begin{enumerate}
  \item Did you include the code, data, and instructions needed to reproduce the main experimental results (either in the supplemental material or as a URL)?
    \answerNo{The data and the experimental details are provided in Section~\ref{sec:exp} and the supplementary material. We will release the code once accepted.}
  \item Did you specify all the training details (e.g., data splits, hyperparameters, how they were chosen)?
    \answerYes{It is illustrated in Section~\ref{sec:exp} in the paper and in the supplementary material.}
        \item Did you report error bars (e.g., with respect to the random seed after running experiments multiple times)?
    \answerYes{Table~\ref{tab:NC} and Table~\ref{tab:LP} are examples.}
        \item Did you include the total amount of compute and the type of resources used (e.g., type of GPUs, internal cluster, or cloud provider)?
    \answerYes{It is provided in the supplementary material.}
\end{enumerate}

\item If you are using existing assets (e.g., code, data, models) or curating/releasing new assets...
\begin{enumerate}
  \item If your work uses existing assets, did you cite the creators?
    \answerYes{}
  \item Did you mention the license of the assets?
    \answerYes{It is mentioned in the supplementary material.}
  \item Did you include any new assets either in the supplemental material or as a URL?
    \answerNo{We do not release new assets.}
  \item Did you discuss whether and how consent was obtained from people whose data you're using/curating?
    \answerNA{}
  \item Did you discuss whether the data you are using/curating contains personally identifiable information or offensive content?
    \answerNA{}
\end{enumerate}

\item If you used crowdsourcing or conducted research with human subjects...
\begin{enumerate}
  \item Did you include the full text of instructions given to participants and screenshots, if applicable?
    \answerNA{}
  \item Did you describe any potential participant risks, with links to Institutional Review Board (IRB) approvals, if applicable?
    \answerNA{}
  \item Did you include the estimated hourly wage paid to participants and the total amount spent on participant compensation?
    \answerNA{}
\end{enumerate}

\end{enumerate}

\appendix

\section{Supplementary material}


In the supplementary material, we provide (1) the related work; 
(2) the complexity and the correctness of the introduced algorithm; (3) the Union-Find algorithm in a sequential format; (4) additional experimental details, including the introduction of the datasets, and the experimental settings; (5) further experiments, including the evaluation on transferability, the influence of training samples, experiments on other datasets, experiments on other attributes of the model and the limitation of the paper.

\subsection{Related Works}

\textbf{Learning with Persistent Homology.} Based on the theory of algebraic topology~\cite{munkres2018elements}, persistent homology~\cite{edelsbrunner2000topological, edelsbrunner2010computational} extends the classical notion of homology, and can capture the topological structures (e.g., loops, connected components) of the input data 
in a robust~\cite{cohen2007stability} manner.
It has already been used in various deep learning domains including kernel machines~\cite{reininghaus2015stable, kusano2016persistence, carriere2017sliced}, convolutional neural networks~\cite{hofer2017deep, hu2019topology, wang2020topogan, zheng2021topological}, transformers~\cite{zeng2021topological}, connectivity loss~\cite{chen2019topological, hofer2019connectivity}, and graph representation learning~\cite{zhao2020persistence,chen2021topological, yan2021link, zhao2019learning, hofer2020graph, carriere2020perslay, kyriakis2021learning}. Some following works propose persistence-inspired frameworks on other tasks such as knowledge graph completion~\cite{yan2022cycle}.  

\textbf{Neural Algorithm Execution.} 
Many works have studied neural execution in different domains before~\cite{zaremba2014learning, kaiser2015neural, kurach2015neural, reed2015neural, santoro2018relational, yan2020neural}. With the rapid development of GNNs in graph representation learning, learning graph algorithms with GNNs has attracted researchers' attention~\cite{velivckovic2019neural, velivckovic2020pointer, xhonneux2021transfer}. These works exploit GNNs to approximate certain classes of graph algorithms, such as parallel algorithms (e.g., Breadth-First-Search) and sequential algorithms (e.g., Dijkstra). Although the computation of extended persistence diagrams can be written in a sequential-like form, it needs extra steps and considerations. In our framework, we propose different modules to approximate these steps and achieve satisfying practical performance.

\textbf{Accelerating Extended Persistent Homology.} 
In general, computing extended persistent homology relies on the well-known matrix reduction algorithm~\cite{cohen2009extending}. Much effort has been made to accelerate the computation, but it still takes matrix multiplication time \cite{ gudhi:urm,dey2022computational, vcufar2021fast}. For the specific case where the input is a function on a graph $G=(V,E)$, it turns out that one can compute it in $O(|E|\log |V|)$ time \cite{agarwal2006extreme,georgiadis2011data}. Nevertheless, this algorithm remains theoretical, and in practice, often a quadratic $O(|V||E|)$ time algorithm is used for its simplicity \cite{yan2021link}. Recently, some works have been proposed to accelerate the computation in a data-driven manner~\cite{som2020pi, montufar2020can, zhou2022learning, de2022ripsnet}. However, these works try to estimate the persistence image~\cite{adams2017persistence}, a coarsened topological feature rather than the persistence diagram itself, leading to much worse performance in both approximation error and downstream tasks. Compared with previous works, we propose a novel framework that directly predicts extended persistence diagrams on graphs. As shown in the experiment, the proposed model has achieved a satisfying approximation error while remaining a high efficiency as well.

\subsection{Complexity and Correctness of Algorithm~\ref{alg:ext_PD}}

\begin{figure}
    \begin{minipage}{\linewidth}
	\begin{algorithm}[H]
	\begin{multicols}{2}
		\begin{algorithmic}[1]	
		\STATE {\bfseries Input:} graph $G = (V,E)$, filter function $f$.
		\STATE Initialise-Nodes($V, f$)
		\STATE $Q = \text{Sort-Queue}(V)$
		\WHILE{Q is not empty}
		\STATE $u = Q.\text{pop-min}()$
		\FOR{$v \in G.\text{neighbors}(u)$}
		\STATE Relax-Edge($u,v,f$)
		\ENDFOR
		\ENDWHILE
		\end{algorithmic}
		\caption{Sequential algorithm} 
		\vspace{-.15in}
		\label{alg:seq1}
		\end{multicols}
	\end{algorithm}
	\end{minipage}
\vspace{-.15in}
\begin{minipage}{\linewidth}
\vspace{-.15in}
\begin{algorithm}[H]
    \begin{multicols}{2}
	\begin{algorithmic}[1]
	\STATE {\bfseries Input:} filter function $f$, input graph $G = (V,E)$
	\STATE $V, E = \text{sorted} (V, E, f)$
	\STATE $PD_0 = \text{Union-Find}(V, E, f)$, $PD_1 = \{\}$
	\FOR{$i \in V$}
	\STATE $C_i = \{C_{i_j} | (i, j) \in E, f(j) > f(i)\}$, $E_i = E$
    \FOR{ $C_{i_j} \in C_i$}
	\STATE $f(C_{i_j}) = f(i)$, $E_i = E_i - \{(i,j)\} + \{(C_{i_j}, j)\}$
	\ENDFOR
	\STATE $PD_1^i = \text{Union-Find-step}(V + C_i - \{i\}, E_i, f, C_i)$
	\STATE $PD_1 += PD_1^i$
	\ENDFOR
	\STATE {\bfseries Output:} $PD_0$, $PD_1$
	\end{algorithmic}
    \caption{Computation of EPD}
    \vspace{-.2in}
	\label{alg:ext_PD1}
	\end{multicols}
	\vspace{-.15in}
    \end{algorithm}

\end{minipage}
\vspace{-.15in}
\begin{minipage}{\linewidth}
		\begin{algorithm}[H]
		\begin{multicols}{2}
		\begin{algorithmic}[1]
		\STATE {\bfseries Input:} $V$, $E$, $f$, $C_i$ 
		\STATE $PD_1^i = \{\}$
		\FOR{$v \in V$}
		\STATE $v.value = f(v)$, $v.root = v$ 
		\ENDFOR
		\STATE $Q = \text{Sort}(V)$, $Q = Q - \{v|f(v) < f(i)\}$, $G = \{Q, E_Q\}$, where $E_Q = E \cup Q^2$.
		\WHILE{Q is not empty}
		\STATE $u = Q.\text{pop-min}()$
		\FOR{$v \in G.\text{neighbors}(u)$}
		\STATE 
		\STATE $pu, pv = \text{Find-Root}(u), \text{Find-Root}(v)$
		\IF{$pu \neq pv$}
		\STATE $s = argmin(pu.value, pv.value)$
		\STATE $l = argmax(pu.value, pv.value)$
		\STATE $l.root = s$ 
		\IF{$pu \in C_i$ and $pv \in C_i$} 
		\STATE $PD_1^i + \{(u.value, l.value)\}$
		\ENDIF 
		\ENDIF
		\ENDFOR
		\ENDWHILE
		\STATE {\bfseries Function:} $\text{Find-Root} (u)$
		\STATE $pu = u$
		\WHILE {$pu \neq pu.root$}
		\STATE $pu.root = (pu.root).root$, $pu = pu.root$
		\ENDWHILE
		\STATE {\bfseries Return:} $pu$
		\end{algorithmic}
		\vspace{-.2in}
		\caption{Union-Find-step (Sequential)}
		\label{alg:UFs1}
		\end{multicols}
		\vspace{-.15in}
		\end{algorithm}
		\end{minipage}
		\vspace{-.15in}
\end{figure}

In this section, we show the complexity and the correctness of Algorithm~\ref{alg:ext_PD} (which is restated as Algorithm~\ref{alg:ext_PD1} in the supplementary material.).

\subsubsection{Complexity}
The computational complexity of the Union-Find algorithm is $O(|E| \alpha(|E|))$~\cite{cormen2009introduction}, where $\alpha(\cdot)$ is the inverse Ackermann function. Therefore, we need $O(|V||E|\alpha(|E|))$ time to compute an 1D EPD using Algorithm~\ref{alg:ext_PD1}. Note this sequential algorithm is not necessarily the most efficient one. In practice, one may use the quadratic algorithm ($O(|V||E|)$) as in \cite{yan2021link}. We also note that although not formally published, the best known algorithm for EPD computation is quasilinear, $O(|E|\log |V|)$, using the data structure of mergeable trees \cite{agarwal2006extreme,georgiadis2011data}. But this algorithm remains theoretical so far.

\subsubsection{Correctness}
Formally, we restate the theorem below (The theorem is named Theorem 3.1 in the main paper). For a clear statement, we present the standard EPD computation algorithm in Algorithm~\ref{alg:MR}. The detailed description of Algorithm~\ref{alg:MR} is beyond the scope of the paper. We only introduce the needed information, and refer the readers to~\cite{cohen2009extending, edelsbrunner2010computational} for details.

\begin{theorem}
Algorithm~\ref{alg:ext_PD1} outputs the same 1D EPDs as Algorithm~\ref{alg:MR}.
\end{theorem}

As stated in Section 2 and Section 3 in the paper, for an edge (1-simplex) $e \in E$, it is either paired with a vertex or an edge. In the former case, the edge, defined as a negative edge, kills a connected component, and gives rise to a 0D persistence point. In the latter case, the edge, defined as a positive edge (in the ascending filtration), creates a loop during the ascending filtration. The loop will ultimately be killed by another edge during the descending filtration (defined as a positive edge in the descending filtration). Hence the positive edge in the ascending filtration is paired with a positive edge in the descending filtration, and gives rise to a 1D extended persistence point. 
For simplicity, we will call the positive edges in the ascending filtration as \emph{ascending positive edges}, and the positive edges in the descending filtration as \emph{descending positive edges}.

In other words, to compute the 1D EPDs, we can simply find the pairing partner for all positive edges. In the following paragraphs, we show that Algorithm~\ref{alg:ext_PD1} produces the same extended persistence pair as the standard EPD computation algorithm. We first present a definition of the ``thinnest pair":

\textbf{Thinnest pair.} Given a filter function $f:X \rightarrow \mathcal{R}$, the pair of edges $(e_1, e_2)$ with $f(e_1) < f(e_2)$ is defined as the thinnest pair if the following condition is satisfied: (1) there is a cycle $C$ having $e_1$ as the lowest edge, and $e_2$ as the highest edge; (2) for any other cycle with $e_1$ as the lowest edge, if its highest edge $e_2$ satisfies that $f(e_3)\neq f(e_2)$, then $f(e_3)> f(e_2)$. Symmetrically, among all cycles having $e_2$ as the highest edge, $e_1$ is the lowest edge in a cycle such that this lowest value is the highest possible.

\begin{lemma}
\label{lemma:1}
For every ascending positive edge, Algorithm~\ref{alg:ext_PD1} finds its ``thinnest pair".
\end{lemma}

\begin{proof}
Algorithm~\ref{alg:ext_PD1} decomposes the 1D extended persistence pair finding for all edges into pair-finding among all nodes. In particular, for a given node $u$, it uses Algorithm~\ref{alg:UFs1} to find the pair for its upper edges. There are two cases:

\textbf{Case 1.} If the upper edge is an ascending negative edge, then it will kill a connected component, and will not influence the 1D extended persistence pairing. 

\textbf{Case 2.} If the upper edge is an ascending positive edge, it will be paired with the loop once the loop is created in the union-find process. The edge, called $e$, is the lowest edge in the loop, called $C$. Recall that $C$ is also the first loop that appears in the union-find process with $e$ as the lowest edge. Therefore $C$ is guaranteed to contain the highest value which is the lowest possible. According to the definition, this will lead to the ``thinnest pair"\footnote{We note that once a loop appears in the union-find process, it will consist of two different upper edges. Considering that the two upper edges share the same filter value in the ascending filtration, the output persistence point will not change no matter which edge is paired with the loop.}.
\end{proof}

\begin{algorithm}[H]
\caption{The standard EPD computation algorithm}
\begin{multicols}{2}
    \centering
	\label{alg:MR}
	\begin{algorithmic}[1]
		\STATE {\bfseries Input:} filter funtion $f$, input graph $G$
		\STATE $EPD=\{\}$
		\STATE$M=$ build reduction matrix$(f,G)$, where $M$ is a $2m*2m$ binary matrix. 
		\FOR{$j=1$ {\bfseries to} $2m$}
		\WHILE{$\exists k < j$ with $low_M(k)=low_M(j)$}
		\STATE add column $k$ to column $j$
		\ENDWHILE
		\STATE add $(f(low_M(j)),f(j))$ to $EPD$
		\ENDFOR
		\STATE {\bfseries Output:} $EPD$
	\end{algorithmic}
	\vspace{-.15in}
\end{multicols}
\vspace{-.15in}
\end{algorithm}

\begin{lemma}
\label{lemma:2_1}
In the descending filtration of Algorithm~\ref{alg:MR}, an edge $e$ is paired if a loop $C$ has already appeared, with $e$ as its lowest edge.
\end{lemma}

\begin{proof}
Every column/row of the binary matrix $M$ shown in Algorithm~\ref{alg:MR} corresponds to a simplex (node/edge) in the input graph $G$. For simplicity, we replace the index in $M$ with the simplex it represents in the rest of the paper. For an edge $e$, $low_M(e)$ denotes its lowest row $e_1$, with $M[e, e_1] = 1$. After the matrix reduction process, $e$ and $e_1$ will form an extended persistence pair. It has been shown in~\cite{cohen2009extending, yan2021link} that for a loop, its highest edge and lowest
edge form its extended persistence pair. In other words, $e$ and $e_1$ are the lowest and highest edges of the loop they form. Assume that a loop $C$ has already aroused with $e$ as its lowest edge, then there are two cases for the highest edge $e_1$ in $C$:

\textbf{Case 1.} If there does not exist an edge $e_2$, that appears before $e_1$ in the descending filtration, with $low_M(e_2) = e = low_M(e_1)$, then $e$ will be paired with $e_1$ in Algorithm~\ref{alg:MR}.

\textbf{Case 2.} If there exists an edge $e_2$, that appears before $e_1$ in the descending filtration, with $low_M(e_2) = e = low_M(e_1)$, then $e$ will be paired with $e_2$ or even other edges that appears earlier than $e_2$. Among all possibilities, $e$ is paired before $C$ appears in Algorithm~\ref{alg:MR}.

In other words, $e$ will be paired with $e_1$ or before $e_1$ in Algorithm~\ref{alg:MR}.

\end{proof}

\begin{figure*}[btp]
	\centering
	\scalebox{0.35}{
		\includegraphics[width=\columnwidth]{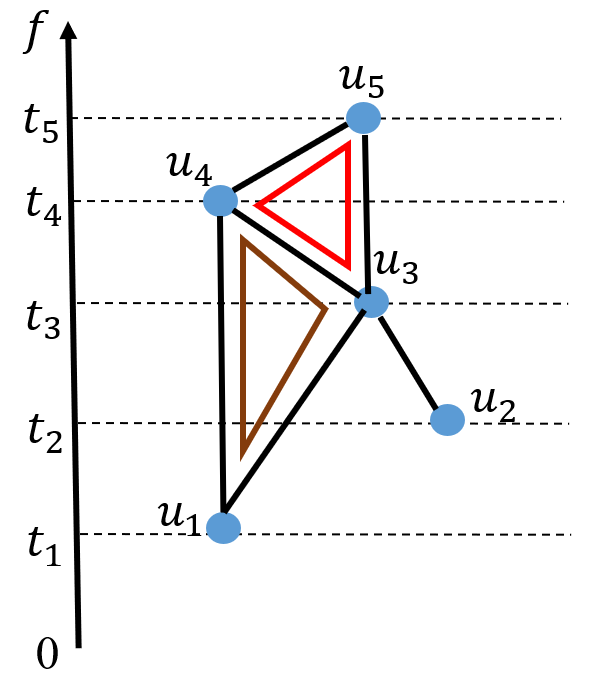}}
	\centering
	\vspace{-.1in}
	\caption{A toy example for Lemma~\ref{lemma:2}.}
	\label{fig:toy}
	\vspace{-.15in}
\end{figure*}

\begin{lemma}
\label{lemma:2}
Algorithm~\ref{alg:MR} finds the ``thinnest pair" for all positive edges.
\end{lemma}

\begin{proof}
During the descending filtration, when an edge $e = u_i u_j$ appears, there are two cases:

\textbf{Case 1.} $e$ is a negative edge, and will kill a connected component. This will not influence the 1D extended persistence pair.

\textbf{Case 2.} $e$ is a positive edge, and will be paired with an ascending positive edge $e_1 = u_a u_b$. Assume that $e_1$ does not construct the ``thinnest pair" with $e$, then there exists an edge $e_2 = u_c u_d$, that forms the ``thinnest pair" with $e$. We can observe a loop $C = u_a \rightarrow u_c u_d \rightarrow u_b$, in which all edges born earlier than $e$ in the descending filtration, and $e_1$ is the latest edge in the ascending filtration. A toy example is shown in Figure~\ref{fig:toy}, where $e = u_1u_3$, $e_1 = u_4u_5$, $e_2 = u_3u_4$, and $C$ is the red loop. Then according to Lemma~\ref{lemma:2_1}, $e_1$ will be paired no later than $C$ appears. In other words, it has already been paired before $e$ appears. Therefore, the assumption is wrong, and Algorithm~\ref{alg:MR} will find the ``thinnest pair" for all positive edges. 

\end{proof}

According to Lemma~\ref{lemma:1} and Lemma~\ref{lemma:2}, Algorithm~\ref{alg:ext_PD1} will produce the same 1D extended persistence pair as Algorithm~\ref{alg:MR}. Therefore, they output the same 1D EPD.

\subsection{Union-Find Algorithm}
In this section, we rewrite the well-known Union-Find algorithm~\cite{cormen2009introduction} in a sequential format. The algorithm is listed in Algorithm~\ref{alg:UF}. Therefore we can use the proposed framework to estimate $PD_0$.
\begin{center}
	\scalebox{1.0}{
		\begin{minipage}{1.0\linewidth}
		\begin{algorithm}[H]
		\begin{multicols}{2}
		\begin{algorithmic}[1]
		\STATE {\bfseries Input:} $G$ = ($V$, $E$), $f$
		\STATE $PD_0 = \{\}$
		\FOR{$v \in V$}
		\STATE $v.value = f(v)$, $v.root = v$ 				
		\ENDFOR
		\STATE $Q = \text{Sort}(V)$
		\WHILE{Q is not empty}
		\STATE $u = Q.\text{pop-min}()$
		\FOR{$v \in G.\text{neighbors}(u)$}
		\STATE $pu, pv = \text{Find-Root}(u), \text{Find-Root}(v)$
		\IF{$pu \neq pv$}
	    \STATE $s/l = argmin/argmax(pu.value, pv.value)$
		\STATE $l.root = s$
		\STATE $PD_0 + \{(l.value, u.value)\}$
        \ENDIF
		\ENDFOR
		\ENDWHILE
		\STATE {\bfseries Function:} $Find-Root (u)$
		\STATE $pu = u$				
		\WHILE {$pu \neq pu.root$}
		\STATE $pu.root = (pu.root).root$, $pu = pu.root$
		\ENDWHILE
		\STATE {\bfseries Return:} $pu$
		\end{algorithmic}
		\caption{Union-Find (Sequential)}
		\label{alg:UF}
	\end{multicols}
	\vspace{-.15in}
	\end{algorithm}
	\end{minipage}
	}
\end{center}

\subsection{Experimental details}

\subsubsection{Datasets.} 

In this paper, we exploit real-world datasets including:

\begin{enumerate}
	\item Citation networks: Cora, Citeseer, and PubMed~\cite{sen2008collective} are standard citation networks where nodes denote scientific documents and edges denote citation links.
	\item Amazon shopping records: In Photo and Computers~\cite{shchur2018pitfalls}, nodes represent goods, edges represent that two goods are frequently brought together, and the node features are bag-of-words vectors.
	\item Coauthor datasets: In CS and Physics~\cite{shchur2018pitfalls}, nodes denote authors and edges denote that the two authors co-author a paper. 
\end{enumerate}

The detailed statistics are available in Table~\ref{tab:NC_stat}.

\begin{table*}
	\vspace{-0.1 in}
	\centering
	\caption{Statistics of the node classification datasets}
	\label{tab:NC_stat} 
	\scalebox{1.0}{
		\begin{tabular}{lccccc}
			\hline\noalign{\smallskip}
			Dataset & Classes & Nodes & Edges & Features & Avg degree\\
			\noalign{\smallskip}\hline\noalign{\smallskip}
			Cora & 7 & 2708 & 5429 & 1433 & 2.00\\
			Citeseer &  6 & 3327 & 4732 & 3703 & 1.42 \\
			PubMed & 3 & 19717 & 44338 & 500 & 2.25 \\
			CS& 15 & 18333 & 100227 & 6805 & 5.47\\
			Physics & 5 & 34493 & 282455 & 8415 & 8.19\\
			Computers & 10 & 13381 & 259159 & 767 & 19.37\\
			Photo & 8 & 7487 & 126530 & 745 & 16.90\\
			\noalign{\smallskip}
			\hline
			\noalign{\smallskip}
	\end{tabular}}
	\vspace{-0.1 in}
\end{table*}

\subsubsection{Experimental Details}

In this section, we mainly present the experimental settings on neural estimation, as for the setting in downstream graph representation learning tasks, we are consistent with~\cite{zhao2020persistence, yan2021link}.

Following the settings in~\cite{zhao2020persistence, yan2021link}, we extract 2-hop neighborhoods of all the nodes in Cora, Citeseer, PubMed and 1-hop neighborhoods of all the nodes in Photo, Computers, Physics, and CS. In the training process, we only adopt the $W_2$ distance between the predicted diagram and the ground truth diagram as the loss function, while the PIE between the predicted persistence image and the ground truth persistence image only serves as an evaluation metric. 

We adopt Adam as the optimizer with the learning rate set to 0.002 and weight decay set to 0.01. We build a 4-layer GNN framework with dropout set to 0. In the training process, we set the batch size to 10, and the training epoch to 20. In this paper, we also exploit a 2-layer MLP to transform the node embedding obtained by the GNN to the persistence points on edges. In the framework, PRELU is adopted as the activation function, the dimension of hidden layers is set to 32, and the dimension of the output persistence image is 25. All the experiments are implemented with two Intel Xeon Gold 5128 processors,192GB RAM, and 10 NVIDIA 2080TI graphics cards. 

Notice that in the normal computation of Wasserstein distance between PDs, the persistence points can be paired to the diagonal or the persistence points in the other diagram. However, in the experiments, we observe that with this loss function as the supervision, the model may converge to local minima, e.g., all the predicted persistence points are paired to diagonal. Therefore, the predicted points all converge to the diagonal and contain no topological information. To avoid such situations, we force the predicted points to pair with the persistence points in the ground truth diagram rather than the diagonal in the training stage. In the reference stage, we report the normal $W_2$ distance between persistence diagrams, that is, to let the predicted points pair with the diagonal.

\subsubsection{About the assets we used}
Our model is experimented on benchmarks from~\cite{Morris+2020, hu2020open, dwivedi2020benchmarking, sen2008collective, shchur2018pitfalls} provided under MIT license.

\subsection{Additional Experiments}

\subsubsection{Experiments on transferability}
\label{subsub:1}
In this section, we design experiments to evaluate the transferability of PDGNN in terms of different graph structures. Our aim is to evaluate whether the pre-trained model can estimate EPDs on totally unseen graphs. Therefore, we evaluate the models pre-trained on Photo on other datasets, and report the $W_2$ distance between the predicted diagrams and ground truth EPDs. Notice that we only use Ollivier-Ricci curvature~\cite{ni2018network} as the filter function. The results are shown in Table~\ref{tab:trans}. 

\begin{table}
	\centering
	\caption{Transferability in terms of different graph structures ($W_2$ distance.)}
	\label{tab:trans} 
	\scalebox{1.0}{
		\begin{tabular}{lccccc}
			\hline\noalign{\smallskip}
			Pre-train & Cora & Citeseer & PubMed & Photo & Computers\\
			\noalign{\smallskip}\hline\noalign{\smallskip}
			Pre-train & 0.392 & 0.279& 0.444 & 0.379& 0.404\\
			Fine-tune & 0.348 & 0.259 & 0.360& 0.380 & 0.381\\
			\noalign{\smallskip}
			\hline
			Standard& 0.354&0.267 &0.344 & 0.379& 0.377\\
			\noalign{\smallskip}
			\hline
			\noalign{\smallskip}
	\end{tabular}}
\end{table}

In Table~\ref{tab:trans}, ``Pre-train" is to directly predict the EPDs with the pre-trained model, and ``Fine-tune" is to fine-tune an epoch on the new datasets, and then predict the EPDs. As shown in Table~\ref{tab:trans}, directly predicting the EPDs with the pre-trained model perform comparably with the standard settings among datasets. We also observe that with only a one-epoch fine-tuning, the pre-trained model can achieve almost an equal performance compared with the standard setting. It justifies the fine transferability of PDGNN. Therefore, in a totally new environment, instead of training the uninitialized models for many epochs, we can simply fine-tune or even directly use the pre-trained model to estimate EPDs on new graph structures.

\subsubsection{Evaluation on the influence of training samples}
\label{subsub:2}
In this section, we evaluate the influence of training samples on PDGNN. We aim to show that the model can reach an acceptable performance with only a small number of training samples.

Recall that for a given graph, we extract the  $k$-hop neighborhoods of all the nodes and randomly select 80\% of these vicinity graphs to train PDGNN. For a thorough evaluation, we train PDGNN with 5/10/20/40\% vicinity graphs in this experiment and report the $W_2$ distance of persistence diagrams, the PIE of persistence images, and the node classification accuracy (NCA) in Table~\ref{tab:n_sample}. We also visualize the influence in Figure~\ref{fig:n_sample} and Figure~\ref{fig:hard}.

\begin{table}
	\vspace{-.15in}
	\centering
	\caption{Influence of training samples on PDGNN}
	\label{tab:n_sample} 
	\scalebox{0.9}{
		\begin{tabular}{|l|ccc|ccc|ccc|}
			\hline\noalign{\smallskip}
			Dataset & \multicolumn{3}{c|}{Cora} &\multicolumn{3}{c|}{Citeseer} & \multicolumn{3}{c|}{PubMed} \\
			Proportion & $W_2$ & PIE & NCA & $W_2$ & PIE & NCA & $W_2$ & PIE & NCA \\
			\noalign{\smallskip}\hline\noalign{\smallskip}
			5\% & 0.391 & 2.51e-3 & 81.3$\pm$0.6 & 0.273 & 3.12e-3 & 70.0$\pm$0.7&0.330 & 4.35e-3 & 78.0$\pm$0.4\\
			10\% & 0.358 & 1.88e-3 & 81.6$\pm$0.7 &0.231 & 3.01e-3 & 70.5$\pm$0.5 &0.300 & 2.36e-3 & 78.5$\pm$0.4 \\
			20\% & 0.318 & 6.99e-4 & 81.8$\pm$0.8 & 0.227 & 1.63e-3 & 70.6$\pm$0.5 & 0.278 & 1.03e-3 & 78.3$\pm$0.3 \\
			40\% & 0.286 & 9.79e-4& 81.6$\pm$0.6& 0.208 & 9.98e-4 & 70.9$\pm$0.6 & 0.255 & 1.34e-3 & 78.8$\pm$0.5 \\
			80\% & 0.241 & 4.75e-4 & 82.0$\pm$0.5 & 0.183 & 4.43e-4 & 70.8$\pm$0.5& 0.256 & 8.95e-4 & 78.7$\pm$0.6 \\
			\noalign{\smallskip}
			\hline
			\noalign{\smallskip}
	\end{tabular}}
\end{table}

\begin{figure*}[btp]
	\centering
	\scalebox{0.8}{
		\includegraphics[width=\columnwidth]{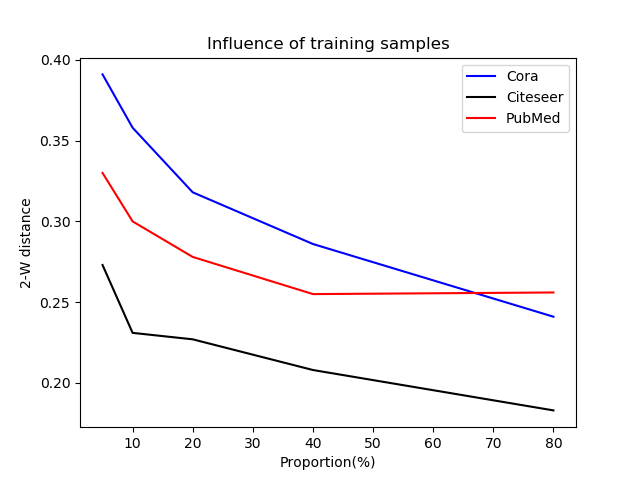}}
	\centering
	\vspace{-.1in}
	\caption{Influence of training samples.}
	\label{fig:n_sample}
	\vspace{-.15in}
\end{figure*}

\begin{figure*}[btp]
	\centering
	\subfigure[]{
		\begin{minipage}[t]{0.45\linewidth}
			\centering
			\includegraphics[width=\columnwidth]{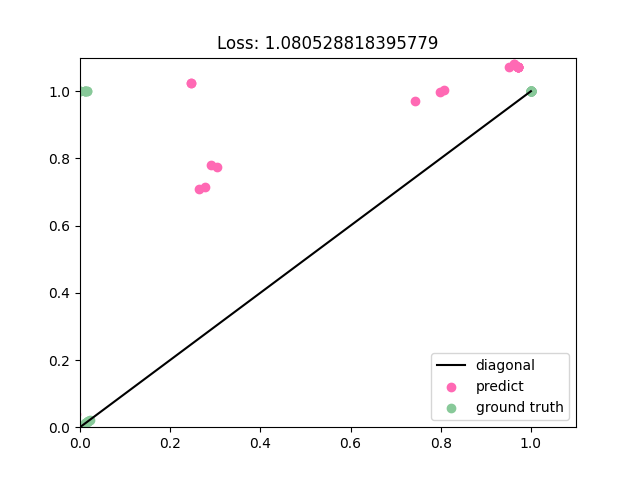}
		\end{minipage}
	}%
	\subfigure[]{
		\begin{minipage}[t]{0.45\linewidth}
			\centering
			\includegraphics[width=\columnwidth]{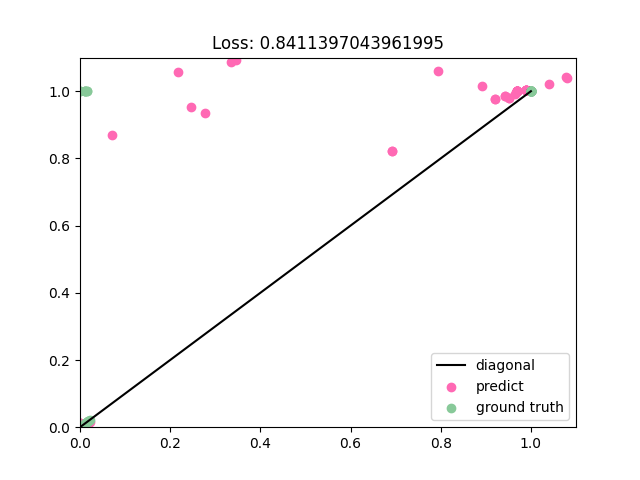}
		\end{minipage}%
	}%
	\\
	\subfigure[]{
		\begin{minipage}[t]{0.45\linewidth}
			\centering
			\includegraphics[width=\columnwidth]{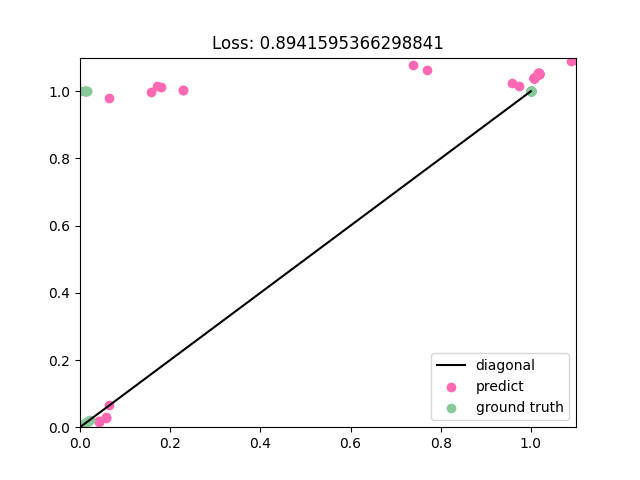}
		\end{minipage}%
	}%
	\subfigure[]{
		\begin{minipage}[t]{0.45\linewidth}
			\centering
			\includegraphics[width=\columnwidth]{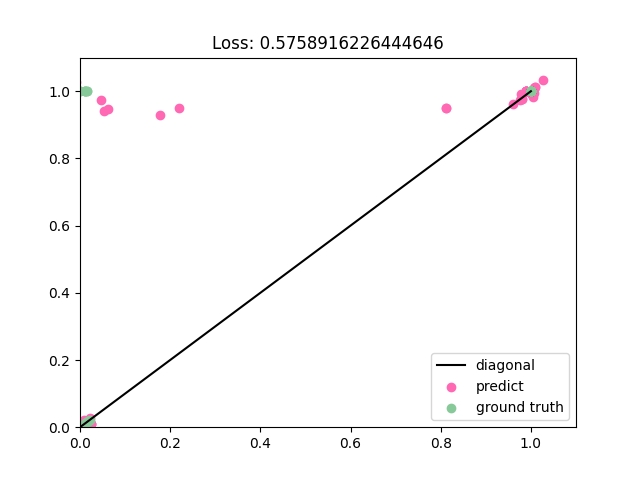}
		\end{minipage}
	}%
	\centering
	\vspace{-.1in}
	\caption{Visualization on the influence of training samples. We select a vicinity graph in Cora with Ollivier-Ricci curvature as the filter function, and plot the influence of training samples on the $W_2$ distance (loss) of EPDs. (a), (b), and (c) denote the prediction of PDGNN with 5/10/20\% training samples, (d) denotes the prediction of PDGNN with the standard setting.}
	\label{fig:hard}
	\vspace{-.15in}
\end{figure*}

As shown in Figure~\ref{fig:n_sample}, the training error tends to converge as the training samples gradually increase. Considering that the $W_2$ distance and PIE cannot directly reflect the learning power as NCA does, we select a vicinity graph in Cora which is hard for PDGNN to learn and visualize in Figure~\ref{fig:hard}. As shown in the figure, as the number of training samples increases, we find that PDGNN can gradually capture the ground truth persistence points in the up y-axis and the up-right diagonal with much less noise. The number of training samples may help the model learn the hard samples better.

We also observe that in Table~\ref{tab:n_sample}, PDGNN reaches a comparable performance on NCA with much fewer training samples. The observation shows that a little perturbation on the persistence image will not influence its structural information very much. 

Combining the observation in Section~\ref{subsub:1} and Section~\ref{subsub:2}, we can safely conclude that our model can be easily generalized to other frameworks. PDGNN does not need many training samples to reach an acceptable performance, and it can be easily transferred to totally unseen graphs.

\subsubsection{Experiments on graph classification datasets.}

In the experiment part, we only consider predicting EPDs of the $k$-hop neighborhoods of the original graphs. Even if these vicinity graphs can be large and dense, there can be structural differences between these vicinity graphs and other real-world graphs. In this section, we do further experiments on graph classification datasets, in which we approximate the EPDs of the real-world graphs rather than the vicinity graphs. We exploit various datasets from the TU Dortmund University~\cite{Morris+2020}, benchmarking-GNN~\cite{dwivedi2020benchmarking}, and OGB~\cite{hu2020open}. The detailed information of these datasets and the approximation error are all available in Table~\ref{tab:GC}.

\begin{table*}
	\centering
	\caption{Statistics and approximation error on the graph classification datasets}
	\label{tab:GC} 
	\scalebox{1.0}{
		\begin{tabular}{lccccc}
			\hline\noalign{\smallskip}
			Dataset & Graphs & Avg Nodes & Avg Edges & $W_2$ & PIE\\
			\noalign{\smallskip}\hline\noalign{\smallskip}
			MUTAG & 188 & 17.9 & 39.6 & 0.300 & 3.06e-4\\
			ENZYMES &  600 & 32.6 & 124.3 & 0.299 & 3.72e-3 \\
			PROTEINS & 1113 & 39.1 & 145.6 & 0.194 & 8.30e-4 \\
			COLLAB & 5000 & 74.5 & 4914.4 &  0.346 & 3.25e-2\\
			IMDB-BINARY & 1000 & 19.8 & 193.1 & 0.176 & 4.13e-4\\
			REDDIT-BINARY & 2000 & 429.6 & 995.5 & 0.383 & 1.92e-4\\
			ZINC (subset) & 12000 & 23.2 & 49.8 & 0.089 & 1.52e-5\\
			OGBG-MolHIV & 41127 & 25.5 & 27.5 & 0.104 & 4.96e-5\\
			\noalign{\smallskip}
			\hline
			\noalign{\smallskip}
	\end{tabular}}
\end{table*}

Notice that we do not add Ollivier-Ricci curvature as the filter function here, because computing the filter function on all the graphs will bring too much computational cost. Comparing the results from Table~\ref{tab:GC} and the results on vicinity graphs, we observe that the performance on graph classification datasets is slightly worse than the performance on vicinity graphs. This may be due to the fact that in graph classification datasets, the training samples can be very small, e.g., there are only 188 graphs in MUTAG, therefore the training is under-fit. On the contrary, the satisfying approximation quality on OGBG-MolHIV and ZINC can be due to their large number of training samples. 

To evaluate the results more clearly, we also visualize some selected examples in Figure~\ref{fig:GC}. As shown in the figure, in most situations, PDGNN can well estimate the EPDs on these graphs, and the $W_2$ distance around 0.3 is generally an acceptable result.

\begin{figure*}[btp]
	\centering
	\subfigure[]{
		\begin{minipage}[t]{0.45\linewidth}
			\centering
			\includegraphics[width=\columnwidth]{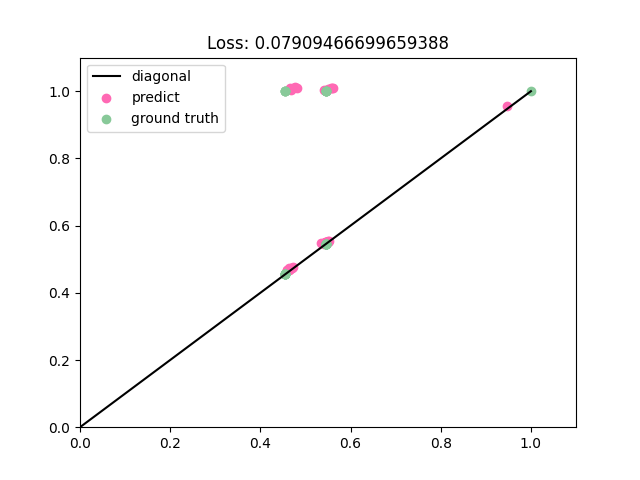}
		\end{minipage}
	}%
	\subfigure[]{
		\begin{minipage}[t]{0.45\linewidth}
			\centering
			\includegraphics[width=\columnwidth]{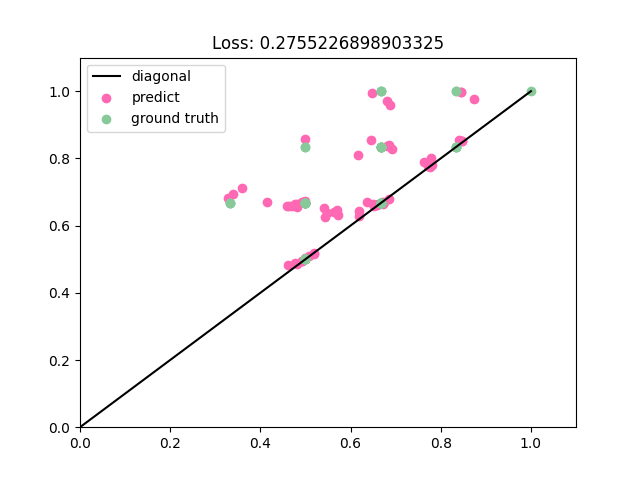}
		\end{minipage}%
	}%
	\\
	\subfigure[]{
		\begin{minipage}[t]{0.45\linewidth}
			\centering
			\includegraphics[width=\columnwidth]{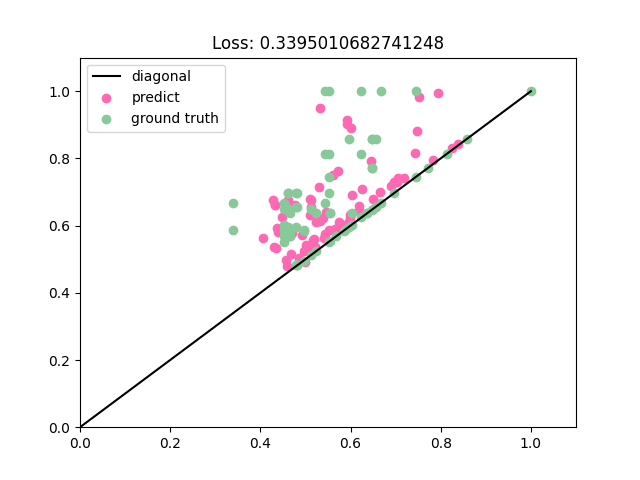}
		\end{minipage}%
	}%
	\subfigure[]{
		\begin{minipage}[t]{0.45\linewidth}
			\centering
			\includegraphics[width=\columnwidth]{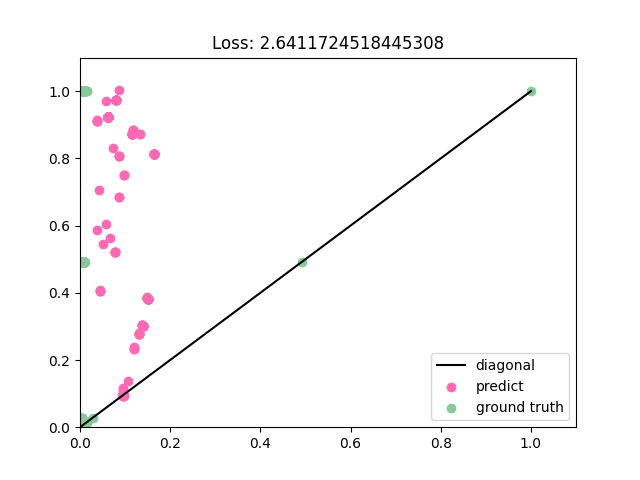}
		\end{minipage}
	}%
	\centering
	\vspace{-.1in}
	\caption{Visualization of graph classification samples. We select samples from IMDB-BINARY, PROTEINS, ENZYMES, and REDDIT-BINARY, respectively, and report the $W_2$ distance (loss).}
	\label{fig:GC}
	\vspace{-.15in}
\end{figure*}

\subsubsection{Why not directly approximating PIs}

We believe directly estimated PIs will lose important structural information that can be crucial for downstream tasks. PI is only an approximation of the persistence diagram. The L2 distance between PIs does not accurately reflect the true Wasserstein distance between diagrams. Therefore, using an L2-distance-based loss to directly learn the PI may lead to the loss of important structural information carried by a diagram. An example is provided in Figure~\ref{fig:direct}. For a sample vicinity graph from Cora, we compare the ground truth PI (computed from the ground truth diagram), the PI computed from the diagram estimated by our method PDGNN, and the PI directly estimated by GIN.  Both estimated PIs have similar L2 distances from the ground truth PI. But we observe that the PI estimated by PDGNN has a very similar spatial distribution to the ground truth PI. This structural property, however, is not preserved by the directly estimated PI. Such loss of structural information of directly estimated PIs, although not captured by the L2 error, partially explains their worse representation power. In Table~\ref{tab:NC} and Table~\ref{tab:LP} in the main paper, the directly estimated PIs (PEGN(GIN\_PI) and TLC-GNN(GIN\_PI)) perform worse in the downstream task. 

This choice of estimating diagrams instead of PIs is a part of the overarching theme of our paper. Note that the main contribution of our paper is to transfer a complicated and uncontrollable learning process to a controllable process with algorithmic insight. This general principle also applies to our learning algorithm. We decompose the diagram computation algorithm into sub-algorithms, which can be approximated well by a GNN. This is the reason PDGNN approximates the diagrams much better than other baselines in Table~\ref{tab:app_nc} in the main paper.

\begin{figure*}[btp]
	\centering
	\subfigure[]{
		\begin{minipage}[t]{0.33\linewidth}
			\centering
			\includegraphics[width=\columnwidth]{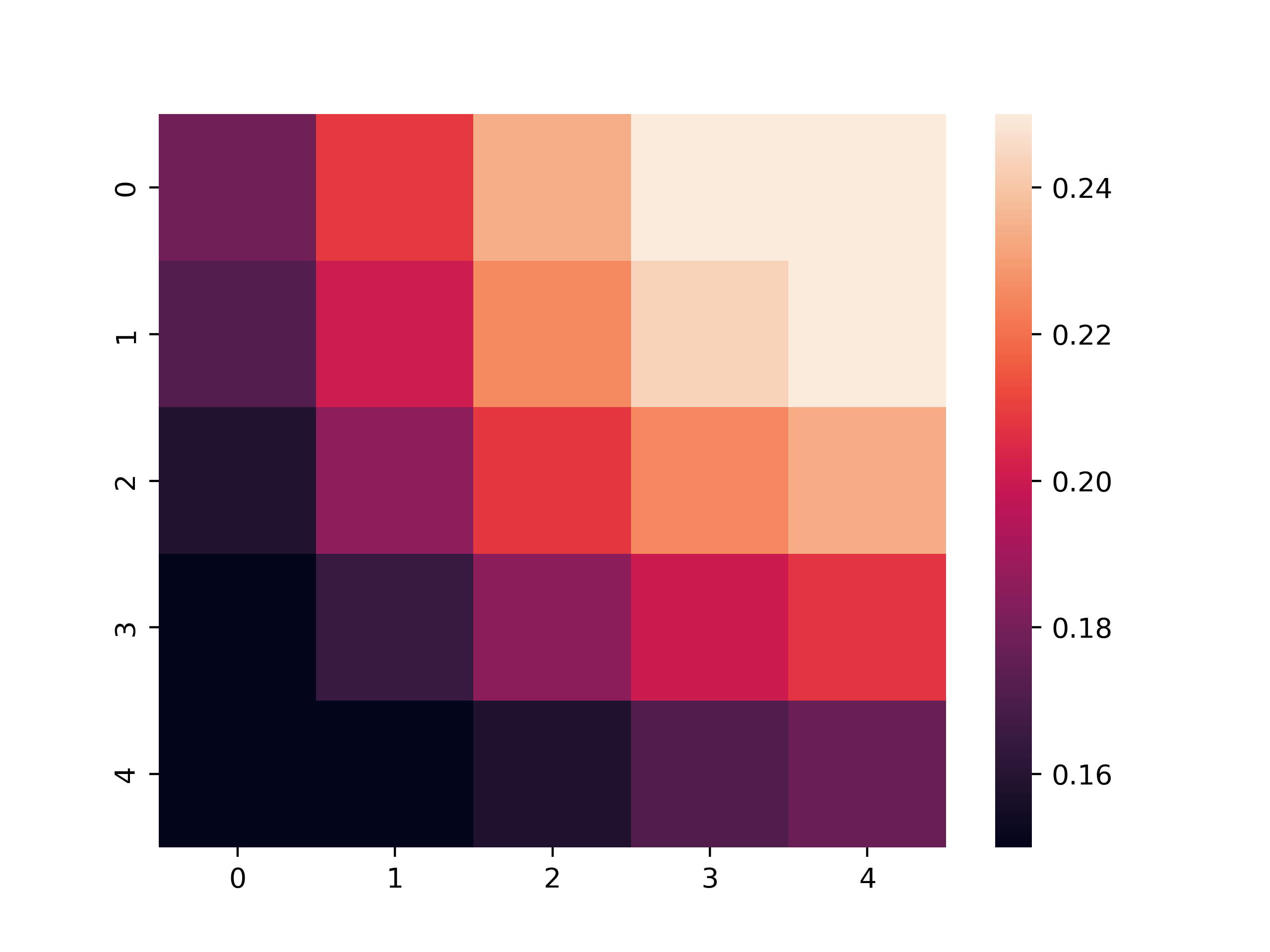}
		\end{minipage}
	}%
	\subfigure[]{
		\begin{minipage}[t]{0.33\linewidth}
			\centering
			\includegraphics[width=\columnwidth]{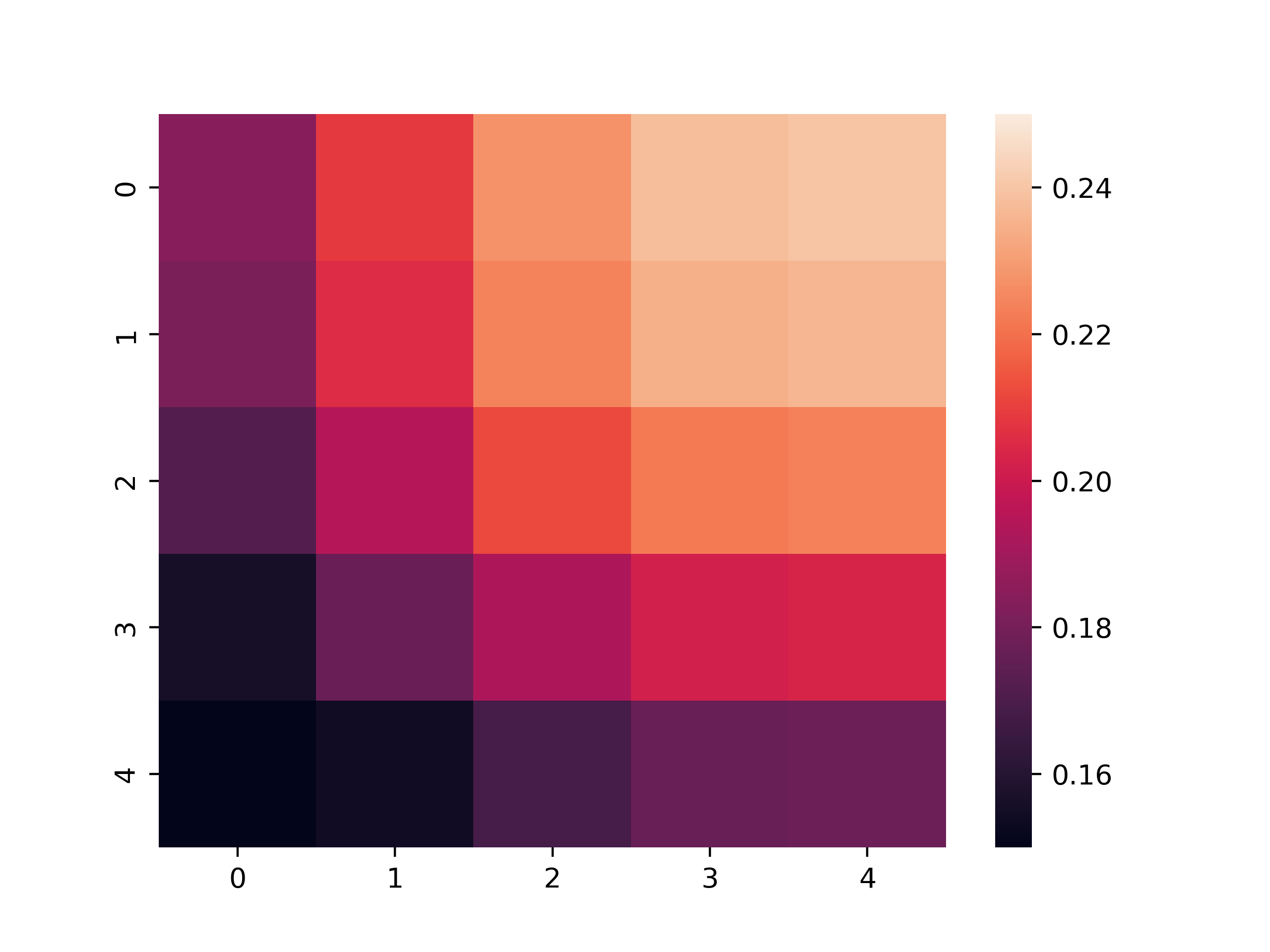}
		\end{minipage}%
	}%
	\subfigure[]{
		\begin{minipage}[t]{0.33\linewidth}
			\centering
			\includegraphics[width=\columnwidth]{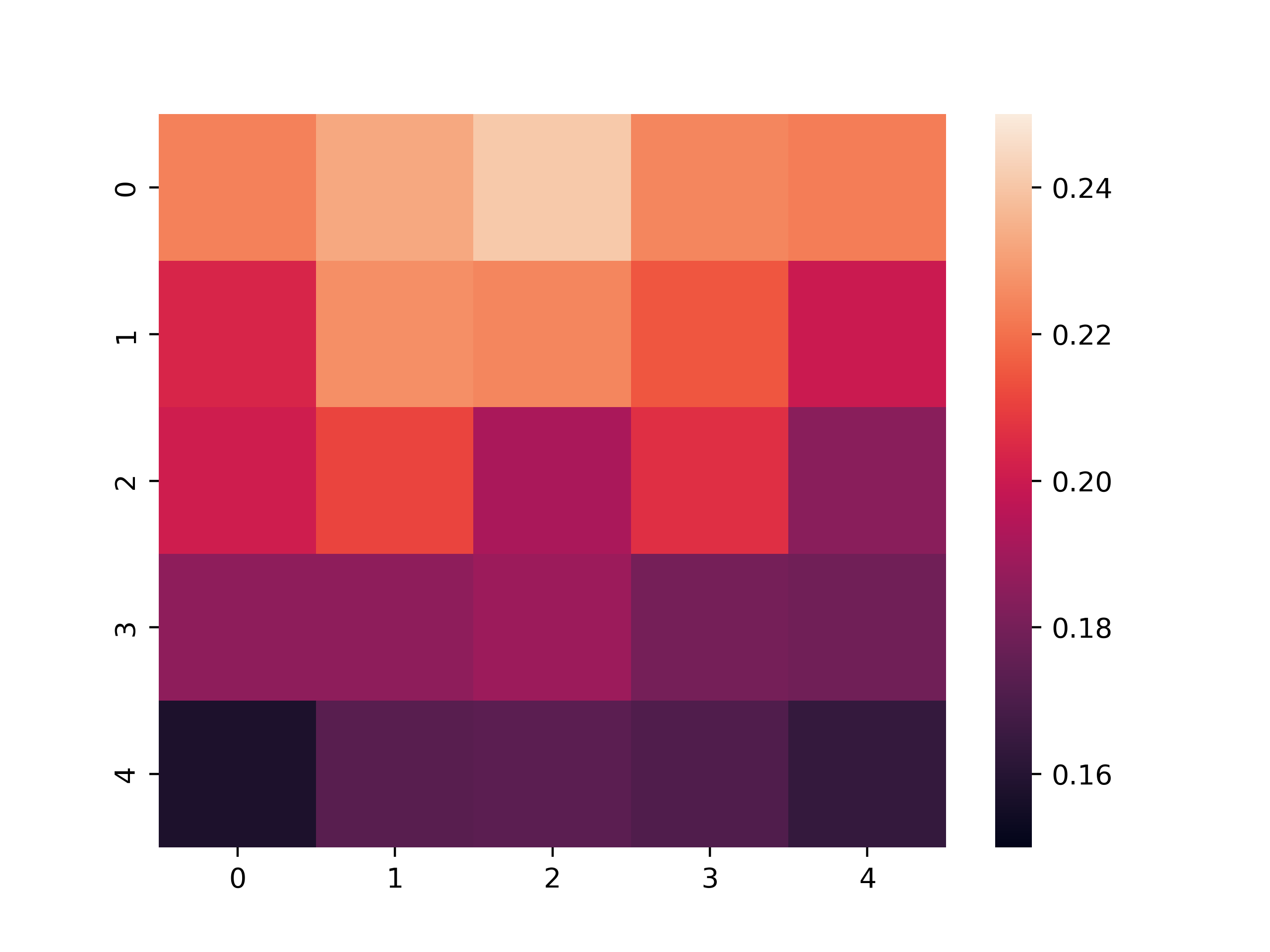}
		\end{minipage}%
	}%
	\centering
	\vspace{-.1in}
	\caption{Examples to explain why not directly approximating PIs.}
	\label{fig:direct}
	\vspace{-.05in}
\end{figure*}

\subsubsection{Experiments on large and sparse datasets.}

\begin{table*}
	\vspace{-0.05 in}
	\centering
	\caption{Experiments on large and sparse datasets.}
	\label{tab:ls} 
	\scalebox{1.2}{
		\begin{tabular}{lccc}
			\hline\noalign{\smallskip}
			Dataset & Cora & Citeseer & PubMed \\
			\hline\noalign{\smallskip}
			Node & 2485 & 2120 & 19717\\	
		    Edge & 5069 & 3679 & 44324\\	\noalign{\smallskip}\hline\noalign{\smallskip}
			Fast~\cite{yan2021link} & 0.184 & 0.068 & 1.816\\
			
			Gudhi~\cite{gudhi:urm} & 0.045 & 0.023 & 1.696 \\
			PDGNN & \textbf{0.006} & \textbf{0.005} & \textbf{0.007} \\
			\noalign{\smallskip}
			\hline
			\noalign{\smallskip}
	\end{tabular}}
	\vspace{-0.05 in}
\end{table*}

In the experiments in the main paper, the input is the $k$-hop vicinity graphs. On citation graphs, the vicinity graph remains small. On these small graphs, the exact sequential algorithm like Gudhi has less overhead, and thus is unsurprisingly faster. 

Indeed, on large and sparse graphs, our method outperforms strong baselines like Gudhi significantly. In Table~\ref{tab:ls}, we compare the running time (in seconds) on popular citation networks including Cora, Citeceer, and PubMed. For each graph, we run experiments on the largest connected subgraph. We also report the number of nodes/edges of the selected subgraph.

\subsubsection{Experiment on the choice of filter functions/other graph metrics.}

\begin{table*}
	\vspace{-0.05 in}
	\centering
	\caption{Experiments on the choice of filter functions.}
	\label{tab:filter} 
	\scalebox{1.1}{
		\begin{tabular}{|l|cc|cc|cc|}
			\hline\noalign{\smallskip}
			Dataset & \multicolumn{2}{c|}{Cora} & \multicolumn{2}{c|}{Citeseer} & \multicolumn{2}{c|}{PubMed} \\
			\hline\noalign{\smallskip}
			Filter & clustering & centrality & clustering & centrality & clustering & centrality\\		\noalign{\smallskip}\hline\noalign{\smallskip}
			\multicolumn{7}{|c|}{Evaluation on approximation error} \\
			\noalign{\smallskip}\hline\noalign{\smallskip}
			$W_2$ & 0.392 & 0.332 & 0.178 & 0.237 & 0.267 & 0.322 \\
			PIE & 1.53e-3 & 4.14e-4 & 7.17e-4 & 4.65e-4 & 2.5e-3 & 3.75e-4\\
			\noalign{\smallskip}\hline\noalign{\smallskip}
			\multicolumn{7}{|c|}{Evaluation on Time (s)} \\
			\noalign{\smallskip}\hline\noalign{\smallskip}
			Fast~\cite{yan2021link} & 2.10 & 2.21 & 1.16 & 1.31 & 38.07 & 39.05\\
			
			Gudhi~\cite{gudhi:urm} & 0.98 & 1.00 & 0.59 & 0.63 & 16.79 & 16.24 \\
			PDGNN & 11.25 & 11.30 & 13.61 & 13.62 & 66.19 & 67.29 \\
			\noalign{\smallskip}
			\hline
			\noalign{\smallskip}
	\end{tabular}}
	\vspace{-0.05 in}
\end{table*}

In Table~\ref{tab:filter}, we set degree centrality and clustering coefficient as the filter function, follow the settings in Table~\ref{tab:app_nc} and report the approximation error on Cora, Citeseer, and PubMed. We also report computation time following the setting in Table~\ref{tab:time}. The only difference is that below we report the time to generate all vicinity graphs (rather than 1000 graphs as in Table~\ref{tab:time}).

We observe that (1) the filter function only has a minor influence on inference/computation speed, for both the sequential algorithm and ours; (2) the filter function does influence the approximation error. The reason is that different filter functions have different ranges; functions with larger ranges tend to have larger approximation errors, especially on PIE. This is another evidence that the distance function on PIs is not very robust for learning. 

\subsubsection{Experiments on the threshold value of average node/edge to decide which method is the fastest to compute/estimate EPDs.}

\begin{table*}
	\vspace{-0.05 in}
	\centering
	\caption{Experiments on the threshold value.}
	\label{tab:thres1} 
	\scalebox{0.85}{
		\begin{tabular}{lccccccccccc}
			\hline\noalign{\smallskip}
			Node & 80 & 84 &88 & 92 & 96 & 100 & 104 & 108 & 112 & 116 & 120 \\
Edge & 515 & 585 & 660 & 713 & 759 & 820 & 943 & 1012 & 1060 & 1152 & 1231\\
\hline\noalign{\smallskip}
Fast~\cite{yan2021link} & 6.8e-3 & 8.0e-3 & 8.9e-3 & 9.6e-3 & 1.0e-2 & 1.1e-2 & 1.3e-2 & 1.4e-2 & 1.4e-2 & 1.6e-2 & 1.7e-2 \\
Gudhi~\cite{gudhi:urm} & \textbf{2.5e-3} & \textbf{3.2e-3} & \textbf{3.6e-3} & \textbf{3.9e-3} & \textbf{4.0e-3} & 4.8e-3 & 5.5e-3 & 6.0e-3 & 6.4e-3 & 6.6e-3 & 6.8e-3 \\
PDGNN & 4.5e-3 & 4.5e-3 & 4.6e-3 & 4.6e-3 & 4.6e-3 & \textbf{4.6e-3} & \textbf{4.7e-3} & \textbf{4.7e-3} & \textbf{4.7e-3} & \textbf{4.7e-3} & \textbf{4.8e-3} \\

			\noalign{\smallskip}
			\hline
			\noalign{\smallskip}
	\end{tabular}}
	\vspace{-0.05 in}
\end{table*}

To find the threshold, we use the well-known Stochastic Block Model (SBM)~\cite{holland1983stochastic} to generate synthetic graphs. We set the number of nodes in these synthetic graphs from 200 to 300, with 10 as the step. In these graphs, we randomly generate 5 different clusters, and set the probability of edges intra-cluster to 0.4, and the probability of edges inter-cluster to 0.1. In this way, we can obtain 11 graphs with different nodes and edges. We set node degree as the filter function, and add experiments on the largest connected components of these 11 graphs. The information of the selected connected graphs and the running time (second) are listed in Table~\ref{tab:thres1}. As shown in the Table, the threshold is around 100 nodes / 820 edges.

\begin{table*}
	\vspace{-0.05 in}
	\centering
	\caption{Experiments on the threshold value.}
	\label{tab:thres2} 
	\scalebox{0.85}{
		\begin{tabular}{lccccccccccc}
			\hline\noalign{\smallskip}
			Node & 100 & 100 & 100 & 100 & 100 & 100 & 100 & 100 & 100 & 100 & 100\\
Edge & 489 & 529 & 595 & 652 &766 & 842 & 968 & 1011 & 1082 & 1231 & 1307 \\
\hline\noalign{\smallskip}
Fast~\cite{yan2021link} & 7.0e-3 & 7.4e-3 & 8.3e-3 & 9.0e-3 & 1.1e-2 & 1.2e-2 & 1.3e-2 & 1.4e-2 & 1.4e-2 & 1.6e-2 & 1.7e-2 \\
Gudhi~\cite{gudhi:urm} & \textbf{2.8e-3} & \textbf{2.9e-3} & \textbf{3.0e-3} & \textbf{4.1e-3} & \textbf{4.2e-3} & 5.1e-3 & 5.5e-3 & 5.9e-3 & 6.2e-3 & 6.3e-3 & 6.7e-3 \\
PDGNN & 4.1e-3 & 4.1e-3 & 4.2e-3 & 4.2e-3 & 4.3e-3 & \textbf{4.4e-3} & \textbf{4.7e-3} & \textbf{4.7e-3} & \textbf{4.8e-3} & \textbf{4.8e-3} & \textbf{4.8e-3}\\

			\noalign{\smallskip}
			\hline
			\noalign{\smallskip}
	\end{tabular}}
	\vspace{-0.05 in}
\end{table*}

We also evaluate the influence of density. We fix the node number of the SBM model to 250, and set the probability of edges intra-cluster from 0.5 to 0.7, and the probability of edges inter-cluster from 0.05 to 0.15. The steps for intra-cluster and inter-cluster are 0.02 and 0.01, respectively. In this way, we can obtain 11 graphs with the same nodes and different edges. We set node degree as the filter function, and add experiments on the largest connected components of these 11 graphs. The information of the selected connected graphs and the running time (second) are also listed in Table~\ref{tab:thres2}. As shown in the Table, the threshold is around 100 nodes / 766 edges.

\subsubsection{Limitation of the paper.}
First, in certain cases like Figure~\ref{fig:GC} (d), the model only captures a tendency of the EPD. This can be because that the distribution of the EPD of the selected graph is seldom in the training samples. Therefore, it is hard for the model to estimate these EPDs correctly.

Second, topological features are just one side of the data. In many cases, only using topological features such as EPDs to represent the information of graphs is not enough. A better way is to introduce other information such as the semantic information of graphs as complementary.

\end{document}